\documentclass{article}

    \usepackage{graphicx, tikz}
        \usetikzlibrary{arrows.meta, positioning}
        
    \usepackage{mathtools, amsthm, amsfonts, amsmath, amssymb, nicefrac}
    \usepackage{enumitem}
        \setlist[itemize,enumerate]{
          nosep,
          topsep=0pt,
          partopsep=0pt,
          parsep=0pt,
          itemsep=0pt,
          leftmargin=1.5em,
          labelsep=0.4em
        }
    \usepackage{fullpage}
    \usepackage[numbers,sort&compress]{natbib}
    \usepackage{hyperref}
    \usepackage{cleveref} %TC: Keep this last!!!

    \newtheorem{corollary}{Corollary}
    \newtheorem{definition}{Definition}
    \newtheorem{proposition}{Proposition}
    \newtheorem{example}{Example}
    \newtheorem{lemma}{Lemma}
    
    \newtheorem{theorem}{Theorem}

    \newcommand{\cV}{\mathcal{V}}
    
    \newcommand{\cX}{\mathcal{X}}
    \newcommand{\cY}{\mathcal{Y}}
    
    \newcommand{\E}{\mathbb{E}}

    \newcommand{\N}{\mathbb{N}}

    \newcommand{\R}{\mathbb{R}}
    
     \newcommand{\e}{\varepsilon}

    \newcommand{\brb}[1]{\bigl(#1\bigr)}
    \newcommand{\Brb}[1]{\Bigl(#1\Bigr)}
    \newcommand{\bbrb}[1]{\biggl(#1\biggr)}
    \newcommand{\Bsb}[1]{\Bigl[#1\Bigr]}
    \newcommand{\bbsb}[1]{\biggl[#1\biggr]}
    \newcommand{\lcb}[1]{\left\{#1\right\}}
    \newcommand{\bcb}[1]{\bigl\{#1\bigr\}}

    \newcommand{\bbfl}[1]{\biggl\lfloor#1\biggr\rfloor}
    \newcommand{\labs}[1]{\left\lvert#1\right\rvert}
    \newcommand{\babs}[1]{\bigl\lvert#1\bigr\rvert}

    \DeclareMathOperator*{\argmax}{argmax}

    \DeclareSymbolFont{extraup}{U}{zavm}{m}{n}
    \DeclareMathSymbol{\clubsuit}{\mathalpha}{extraup}{84}
    \DeclareMathSymbol{\spadesuit}{\mathalpha}{extraup}{81}
    \DeclareMathSymbol{\varheartsuit}{\mathalpha}{extraup}{86}
    \DeclareMathSymbol{\vardiamondsuit}{\mathalpha}{extraup}{87}

    \newcommand{\mc}[1]{\mathcal{#1}}

    \newcommand{\set}[1]{\left\{#1\right\}}
    
    \crefname{appendix}{Appendix}{Appendices}
    \Crefname{appendix}{Appendix}{Appendices}

    \title{On the Cost and Benefit of Chain of Thought:\\
    A Learning-Theoretic Perspective}

\author{
    Yue Zhang\textsuperscript{\rm (1)},
    Zhiyi Dong\textsuperscript{\rm (1)},
    Tommaso Cesari\textsuperscript{\rm (1)},
    Yongyi Mao\textsuperscript{\rm (1)}
    \\[0.5em]
    \textnormal{\textsuperscript{\rm (1)} University of Ottawa, Ottawa, Canada}\\
}
    
\begin{document}
    
    \maketitle
    \begin{abstract}
We develop a learning-theoretic framework for understanding Chain of Thought (CoT).
We model CoT as the interaction between an answer map and a chain rule that generates intermediate questions autoregressively, and define the reasoning risk of a hypothesis under this interaction.
Our first result is a tight canonical decomposition of this risk into two terms with opposing roles: an oracle-trajectory risk (OTR), which captures the benefit of CoT and reduces to a target-domain risk in a domain adaptation problem, and a trajectory-mismatch risk (TMR), which captures the cost of CoT through error accumulation along mismatched reasoning trajectories.
We then show that this cost is unavoidable without structure: if any one of the loss, the hypothesis answer map, or the chain rule lacks stability, the TMR can be arbitrarily large even when the OTR is zero and the hypothesis is uniformly close to the ground truth.
Conversely, under stability, we prove a tight upper bound on the TMR governed by an exact amplification factor that identifies bounded, linear, and exponential error-growth regimes.
Together, these results give a precise theory of when CoT helps, when it hurts, and what controls the transition between the two.
    \end{abstract}

    \section{Introduction}

The rise of large language models (LLMs) has reshaped the landscape of AI technologies.
Large LLMs, such as ChatGPT~\citep{openai_openai_2025}, 
LLaMA~\citep{touvron2023llama}, Claude~\citep{anthropic_claude_2026}, 
and Gemini~\citep{google_gemini-3-pro_2025}, trained on internet-scale data, have demonstrated striking capabilities.
Although these models are trained to predict the next word, one at a time in an autoregressive manner~\citep{vaswani_attention_2023, radford_language_2019, radford_improving_2018, brown_language_2020, zhang_sirens_2025}, they appear to have acquired remarkable abilities to answer user queries.
Most remarkably, there is growing evidence that these models can perform forms of ``reasoning'', although the precise notion of ``reasoning'' remains difficult to formalize.
Accordingly, a large body of work has focused on understanding or improving the reasoning capabilities of LLMs; see, e.g., \citep{gan_beyond_2026, chen_towards_2025, wang_multimodal_2025, li_implicit_2025, pan_survey_2025} and the references therein.

One important reasoning mechanism for LLMs is known as \emph{Chain of Thought}, or \emph{CoT}.
Specifically, it has been observed that when a query is too complex for the LLM to answer correctly, appending an additional instruction, such as ``solve the question step by step'', to the query prompt often allows the LLM to give the correct answer~\citep{wei_emergent_2022, wei_chain--thought_2023, kojima_large_2023}.
That is, encouraging the LLM to output a ``chain of thought'' may enable more powerful reasoning processes.
It is natural to suspect that the effectiveness of CoT is linked to the autoregressive nature of LLMs: generated tokens are fed back to the model as part of its input prompt, and this feedback, unfolding the ``chain of thought'', appears to facilitate reasoning.
At a high level, direct prompting asks the model to answer the original input, whereas CoT prompting first transforms the input into a reasoning-augmented prompt; our goal is to understand, from a learning-theoretic perspective, how and when this transformation improves or degrades performance.
Beyond this high-level intuition, however, relatively little is known theoretically about LLMs under CoT prompting.
Much of the existing theoretical work on CoT has focused on the expressivity of Transformer-based models\footnote{Transformers are the key building blocks of modern LLMs~\citep{vaswani_attention_2023}.} under CoT prompting~\citep{feng_towards_2023, li_chain_2024, merrill_expressive_2024, zhu_reasoning_2025}. 
The most prominent statistical learning-theoretic treatment of CoT to date is due to \citet{joshi_theory_2025}, who establish PAC-learning results for token-level hypothesis classes under 0-1 loss, reflecting the autoregressive structure of token generation.

Our interest in CoT is learning-theoretic. 
Our 
goal is to develop a learning theory that characterizes the structure of the LLM learning problem when the learned LLM is used for inference under CoT prompting.
Such a theory could help explain the benefit of CoT, the price that must be paid to obtain this benefit, and the guarantees and hardness phenomena associated with such learning tasks.

This work takes a step in that direction.
We begin with a general formulation of CoT via a notion we call a \emph{chain rule}, which captures the autoregressive structure of query-answer generation by LLMs.
In this formulation, the CoT answer to a question is treated as a sequence of new questions generated by the chain rule, together with their respective answers obtained from an answer map.
This formulation allows us to define a notion of recoverability for each question, which determines whether the answer to the question can be obtained by CoT prompting.
It also allows us to define the \emph{reasoning risk} of an LLM implementing a hypothesis answer map with respect to the ground-truth answer map.
Although this risk is defined for any test distribution, we are primarily interested in the case where the distribution is supported on what we call ``recoverable'' points\footnote{At a high level, these are prompts whose answers can be recovered by following the chain rule under the ground truth.}, so as to focus on the behavior of the tested LLM on questions whose answers can be recovered by CoT prompting under the ground truth.

In this setting, and under the mild assumption that the loss function defining the risk is a quasimetric, we observe that the reasoning risk decomposes into two terms, which we respectively designate the \emph{trajectory-mismatch risk} (TMR) and the \emph{oracle-trajectory risk} (OTR).
The TMR term captures the risk due to the CoT following the ``wrong trajectory'' because of answering the chain-rule-generated questions using a hypothesis answer map rather than the ground truth, whereas the OTR term captures the risk due solely to the final answer being given by the hypothesis answer map when the CoT trajectory has been correctly guided by the ground truth up to the final question.
These two terms are respectively associated with the \emph{cost} and \emph{benefit} of CoT.

First, the OTR term, which captures the benefit of CoT, can be viewed as the target-domain risk arising in a domain adaptation problem~\citep{ben-david_theory_2010}.
In particular, the chain rule interacting with the ground-truth answer map serves as a transformation of input questions, and the test distribution pushed forward by this transformation can be viewed as the target-domain distribution; the distribution from which examples are drawn to train the LLM is the source-domain distribution.
Thus, OTR minimization coincides with the objective of domain adaptation.
This perspective reveals how CoT prompting can be beneficial: when the push-forward of the test distribution is close to the distribution used for training, the OTR is guaranteed to be small.
In other words, CoT prompting can provide benefit by aligning the test distribution, via push-forward, with the training distribution; this allows the model to answer questions that are more complex than, and very different from, those seen during training.

On the other hand, the TMR term captures an unavoidable cost of CoT prompting: any inconsistency between a hypothesis answer map and the ground truth can be magnified along the mismatched CoT trajectory, contributing to an additional risk.
We characterize the behavior of this TMR term.
Specifically, via a no-free-lunch theorem, we show that without appropriate stability assumptions on the loss function, the hypothesis answer map, and the chain rule, the TMR can be arbitrarily large.
Under such stability assumptions, we then prove an upper bound on the TMR and show that this bound is optimal.
This bound identifies three cost regimes, in which the TMR remains bounded, grows linearly, or grows exponentially with the number of CoT steps.
This result helps explain error-accumulation phenomena observed in recent work on CoT~\citep{gan_rethinking_2025, bachmann_pitfalls_2025, zhu_advchain_2025}.

As a step toward a learning theory of CoT, this work provides a theoretical analysis of both the benefit and cost of CoT from a learning-theoretic perspective.
Our formulation differs from that of \citep{joshi_theory_2025}: while they study token-level hypothesis classes under 0-1 loss, we work over an abstract representation space, allow quasimetric losses, and model CoT through a higher-order autoregressive object, the chain rule.
These perspectives are complementary.

\paragraph{Our contributions.}
We defer additional related work to \Cref{s:related} and summarize our main contributions below.
\begin{itemize}
    \item We formalize CoT as a learning problem, define the reasoning risk, and show that it admits a tight decomposition into two terms: the trajectory-mismatch risk (TMR), which captures the cost of reasoning by measuring error accumulation along mismatched CoT trajectories, and the oracle-trajectory risk (OTR), which captures the benefit of reasoning by reducing to a target-domain risk in a domain adaptation problem.
    
    \item We investigate the necessity of stability for controlling the TMR. We prove a no-free-lunch theorem showing that without simultaneous stability of the loss function, the hypothesis answer map, and the chain rule, the TMR can be arbitrarily large, even when the OTR is zero and the hypothesis is uniformly close to the ground truth (\Cref{thm:main-no-free-lunch}).
    
    \item We prove an optimal upper bound on the TMR under stability (\Cref{thm:stability-TMR,thm:stability-reasoning-risk-tight}). 
    The bound features an \emph{exact} amplification factor describing when CoT errors remain bounded, accumulate linearly, or grow exponentially, and yields an upper bound on the reasoning risk that is tight in \emph{all} stability regimes (\Cref{cor:stability-reasoning-risk} and \Cref{thm:stability-reasoning-risk-tight}).
\end{itemize}
    
    \section{Formalizing Chain of Thought}
    \label{sec:formalizing-cot}

    In this paper, we view an LLM as a machine that maps an input object to an output object.
Concretely, the input may be a sequence of tokens, such as a user question or prompt, and the output may also be a sequence of tokens, such as an answer or response.
The terms ``question'' and ``answer'' should therefore be interpreted broadly: a question may refer to any input object submitted to the model, and an answer to the corresponding output object generated by the model.
This convention does not exclude contextual or multi-round interactions, since any relevant context, conversation history, system prompt, retrieved document, or memory state can be incorporated into the input object itself.
Thus, our formulation studies the map from a single input representation to a single output representation, without restricting what information the input representation may encode.

Modern LLMs typically operate in an autoregressive manner~\citep{openai_openai_2025, deepseek-ai_deepseek-r1_2025, google_gemini-3-pro_2025, abdullah_evolution_2025}: they generate one output token at a time, and each generated token is appended to the input sequence before the next token is generated.
Throughout the paper, we assume that this autoregressive process is deterministic, namely, that the next token at each time step depends functionally on the input sequence and on the output tokens generated so far.
More precisely, if $\cV$ denotes the token vocabulary and $\cV^*$ denotes the space of all finite strings over $\cV$, then an LLM with fixed parameters can be identified with a sequence of functions $F_1,F_2,\ldots$, collectively denoted by $F$, where each function $F_i\colon \cV^i\to\cV$ generates the next token from the current input string of length $i$.
Although this restriction abstracts away from the stochastic decoding used by practical LLMs, it allows us to focus on key aspects of reasoning in LLMs without introducing additional technicalities arising from stochasticity, an interesting topic in its own right.

The autoregressive process specified by $F$ does not run forever: once a special token indicating the END signal is generated, the process terminates.
Thus, one can equivalently identify $F$ with a function mapping $\cV^*$ to $\cV^*$, which, by abuse of notation, we also denote by $F$.
That is, for any input string $x\in\cV^*$, the LLM specified by $F$ generates an output string $F(x)\in\cV^*$.

To answer a user question, presented as a string $x$, the LLM $F$ may be used in two modes.
In the \emph{direct mode}, $x$ is passed to the LLM directly, and the LLM generates the answer string $F(x)$.
In the \emph{CoT mode}, an instruction prompt, for example ``solve the question step by step'', is appended to $x$.
The resulting string, denoted by $x^{\rm CoT}$, is then fed to the LLM, which generates $F(x^{\rm CoT})$ as the answer.
Empirical evidence~\citep{wei_chain--thought_2023, kojima_large_2023, sprague_cot_2025, stechly_chain_2025, wang_towards_2023, wang_self-consistency_2023} indicates that $F(x^{\rm CoT})$ is often more accurate than $F(x)$, suggesting that CoT prompting can elicit useful reasoning behavior, even though this behavior lacks a precise theoretical characterization.
Developing a sound theoretical understanding of LLMs operating in the CoT mode is the key motivation of this work.
To that end, define the function $F^{\rm CoT}\colon \cV^*\to\cV^*$ by $F^{\rm CoT}(x)=F(x^{\rm CoT})$.
The starting point of this work is a formalization of $F^{\rm CoT}$ in relation to $F$.

In our formulation, instead of treating the LLM $F$ as a function mapping $\cV^*$ to $\cV^*$, we consider it as a function mapping $\cX$ to $\cX$, for an appropriate representation space $\cX$.
We impose no structure on $\cX$: it may be the space $\cV^*$ of strings, a Euclidean space in which strings are embedded, or an abstract space of semantic objects.
At this level of abstraction, the LLM $F$ is represented as a function $f\colon\cX\to\cX$, and the corresponding function $F^{\rm CoT}$ associated with the CoT mode is formalized through a ``chain rule'', which we introduce next.
    
    \subsection{Learning setting and reasoning risk}

    Let $\cX$ be a set, and $K$ a positive integer.
    For any $k\in [K]$, let $D_K^{(k)} \colon \cX^{2k-1} \to \cX$;
    we say that the sequence $D_K \coloneqq \brb{ D_K^{(1)}, \dots, D_K^{(K)} }$ is a \emph{($K$-step) chain rule} and that $D_K^{(k)}$ is its $k$-th \emph{step}. 
    For any $f\colon \cX \to \cX$, which we call an \emph{answer map}, we define the sequence of functions
    $Q^{(1)}_{D_K,f}, A^{(1)}_{D_K,f}, \dots, Q^{(K)}_{D_K,f}, A^{(K)}_{D_K,f}\colon \cX \to \cX$
    inductively, as follows: 
    $Q^{(1)}_{D_K,f} \coloneqq D_K^{(1)}$, $A^{(1)}_{D_K,f} \coloneqq f \circ Q^{(1)}_{D_K,f}$,
    and for all $k \in \{2, \dots, K\}$ and $x\in \cX$,
    \[
        Q^{(k)}_{D_K,f}(x)
    \coloneqq
        D_K^{(k)} \Brb{
        x, \mspace{1mu}
        Q^{(1)}_{D_K,f} (x), \mspace{1mu}
        A^{(1)}_{D_K,f} (x), \mspace{1mu}
        \dots,
        Q^{(k-1)}_{D_K,f} (x), \mspace{1mu}
        A^{(k-1)}_{D_K,f} (x)
        },
    \ \ \
        A^{(k)}_{D_K,f}(x)
    \coloneqq
        f \brb{ Q^{(k)}_{D_K,f}(x) },
    \]
    and we say that $Q^{(k)}_{D_K,f}$ (resp., $A^{(k)}_{D_K,f}$) is the \emph{$k$-th question} (resp., \emph{answer}) and that the sequence 
    $\bigl( 
        x,$ 
    $    Q^{(1)}_{D_K,f} (x),$ 
    $    A^{(1)}_{D_K,f} (x),$ 
    $    \dots,$ 
    $    Q^{(K)}_{D_K,f} (x),$ 
    $    A^{(K)}_{D_K,f} (x) 
    \bigr)$ 
    is a \emph{($K$-step) Chain of Thought, or CoT (with initial prompt $x$ with respect to $(D_K, f)$)}.\footnote{%
    This recursive dependence captures the autoregressive nature of CoT: each new question may depend on the original prompt together with all previously generated questions and answers, just as later tokens in an LLM response are generated conditional on the earlier tokens.%
    }
    Furthermore, for any $f\colon \cX \to \cX$, we say that an element $x \in \cX$ is \emph{$(D_K, f)$-recoverable} if $f(x) = A^{(K)}_{D_K, f}(x)$.
    Let also $g\colon \cX \to \cX$ be a fixed unknown answer map, serving as, and hence referred to as, the \emph{ground truth}.
    We now provide an example to clarify these definitions.

    \begin{example}[Chain rule in arithmetic]
    \label{ex:multiply} 
    Let $\cX$ be the subset of all finite strings $x$ from the alphabet $\{ 0, 1, 2, 3, 4, 5, 6, 7, 8, 9, +, \cdot \}$ such that one of the following is true: a) $x$ is a nonempty string of digits, b) $x$ is the concatenation $x_1 x_2 x_3$ of three strings $x_1, x_2, x_3$, where $x_1$ and $x_3$ are nonempty strings of digits, and $x_2$ consists of a single ``$+$'' symbol, c) $x$ is the concatenation $x_1 x_2 x_3$ of three strings $x_1, x_2, x_3$, where $x_1$ and $x_3$ are nonempty strings of digits, and $x_2$ consists of a single ``$\cdot$'' symbol.
    In words, the elements of $\cX$ represent natural numbers as well as additions and multiplications of natural numbers.
    
    Fix the ``ground truth'' $g\colon \cX \to \cX$, i.e., the ideal evaluator that maps each $x\in \cX$ into the corresponding result of the addition/multiplication (if $x$ represents an addition/multiplication) or into $x$ (if $x$ represents a number). 
    This definition of $g$ is used only to specify correctness; it should not be interpreted as saying that the learned hypothesis is trained on all of $\cX$. 
    Indeed, one may take the source distribution to be supported only on elementary expressions, such as single-digit-by-single-digit multiplications, multiplications by $10$, and additions, while the test distribution is supported on more complex multiplication expressions. 
    Consider now a chain rule decomposing the multiplication of a single-digit integer and a double-digit integer in terms of such elementary operations.
    We illustrate the behavior of such a chain rule on the initial prompt\footnote{Here and below, expressions such as $a_1\cdot a_2$ and $a_3+a_4$, where $a_1, \dots, a_4 \in \N$, denote with a slight abuse of notation the corresponding strings in $\cX$.} $x \coloneqq 7 \cdot 26$ (see Appendix~\ref{app:details-example-multiply} for details and explicit definitions):
    \begin{align*}
        Q^{(1)}_{D_4,g} (7 \cdot 26) 
    & 
        = 7 \cdot 2, 
    &
        A^{(1)}_{D_4,g} (7 \cdot 26) 
    &
        = 14, 
    &
        Q^{(2)}_{D_4,g} (7 \cdot 26) 
    &
        = 10 \cdot 14, 
    &
        A^{(2)}_{D_4,g} (7 \cdot 26) 
    &
        = 140,
    \\
        Q^{(3)}_{D_4,g} (7 \cdot 26) 
    & 
        = 7 \cdot 6, 
    &
        A^{(3)}_{D_4,g} (7 \cdot 26) 
    &
        = 42, 
    &
        Q^{(4)}_{D_4,g} (7 \cdot 26) 
    &
        = 140 + 42, 
    &
        A^{(4)}_{D_4,g} (7 \cdot 26) 
    &
        = 182.
    \end{align*}
    More generally, as detailed in Appendix~\ref{app:details-example-multiply}, every one-digit-by-two-digit multiplication expression of the form $d_1\cdot(10d_2+d_3)$ is $(D_4,g)$-recoverable for this chain rule.
    \end{example}
    Two clear advantages of CoT emerge immediately from \Cref{ex:multiply}. 
    First, CoT can enable more \emph{efficient storage}. 
    Without CoT, a hypothesis that answers such test queries directly may need to learn the one-digit-by-two-digit multiplication map. 
    With CoT, it suffices for the hypothesis to be accurate on the elementary subquestions generated by the chain rule: one-digit-by-one-digit multiplications, multiplications by $10$, and additions.
    Of course, in this example, both approaches require negligible memory, but the memory required to multiply larger numbers increases quadratically with the number of digits without CoT, whereas it remains essentially constant with CoT.
    Second, and more importantly, CoT helps a trained model generalize \emph{beyond the training distribution}. Consider, for example, that during training all examples are single-digit-by-single-digit multiplications, multiplications by $10$, and additions. 
    Even if the model never sees one-digit-by-two-digit multiplications, with CoT, it is still possible to solve problems such as $4 \cdot 23$. Abundant empirical evidence also shows that without CoT, language models are generally unable to solve multi-digit multiplication tasks~\citep{nye_show_2021, gambardella_language_2024, bai_why_2025, yang_chain--thought_2025}.
    
    Since out-of-distribution learning appears to be achievable with CoT, it is natural to define a notion of risk with respect to an arbitrary distribution, even if, like in the example, its support is fully disjoint from that of the training distribution.
    For any \emph{testing distribution} $\nu$ on $\cX$, any \emph{loss function} $\ell\colon \cX \times\cX \to [0,\infty)$, and any \emph{hypothesis answer map} (or simply \emph{hypothesis}) $f\colon \cX \to \cX$, we define the \emph{reasoning risk of $f$ (with respect to $D_K, g, \nu, \ell$)} as\footnote{For expectations to be well-defined, we always implicitly assume that $\cX$ is a measurable space, and losses $\ell\colon \cX\times\cX \to [0,\infty)$, answer maps $f,g\colon \cX \to \cX$, and steps $D_K^{(1)}, \dots, D_K^{(K)}$ of chain rules are measurable functions.}
    \[
        R^{D_K}_{\nu,\ell} (f, g)
    \coloneqq
        \underset{X \sim \nu}{\E}
            \bbsb{
                \ell\Brb{
                    A^{(K)}_{D_K,f} (X) , \
                    g(X)
                }
            }.
    \]
    In words, the reasoning risk quantifies the average loss incurred by answering with the hypothesis $f$ after $K$ reasoning steps, rather than answering with the ground truth $g$.

    For any $q\colon \cX\to\cX$ and any distribution $\nu$ on $\cX$, we denote by $q\#\nu$ the push-forward of $\nu$ through $q$.

    \subsection{Risk decomposition}
    
    For the remainder of this work, fix a loss function $\ell\colon \cX \times \cX \to [0,\infty)$ which we assume to be a \emph{quasimetric},\footnote{That is, $\ell$ satisfies the triangle inequality and for all $x\in\cX$, $\ell(x,x)=0$.} a mild regularity assumption that is satisfied by many commonly-used losses (e.g., $0$-$1$, absolute, Hamming, total variation, pinball/quantile losses, \emph{etc}.).  
    Fix also a testing distribution $\nu$ on $\cX$ whose support we assume to consist of $(D_K,g)$-recoverable points, and note that if this minimal structural assumption is dropped, the reasoning risk could be made arbitrarily large.\footnote{See Appendix \ref{sec:large-risks-no-structure} for a more detailed discussion.}
    
    Under these two assumptions on $\ell$ and $\nu$, we establish the following tight canonical decomposition inequality for the reasoning risk, in terms of what we call the \emph{trajectory-mismatch risk} (TMR) and the \emph{oracle-trajectory risk} (OTR):
    for any hypothesis $f\colon\cX\to\cX$,
    \begin{equation}
    \label{eq:canonical-decomposition}
        R_{\nu, \ell}^{D_K}(f,g)
    \le
        \underbrace{
            \underset{X \sim \nu}{\E}
            \bbsb{
                \ell\Brb{
                    f\brb{ Q_{D_K,f}^{(K)}(X) }, \,
                    f\brb{ Q_{D_K,g}^{(K)}(X) }
                }
            }
        }_{\text{trajectory-mismatch risk (TMR)}}
        +
        \underbrace{
            \underset{X \sim \nu}{\E}
            \bbsb{
                \ell\Brb{
                    f \brb{ Q_{D_K,g}^{(K)}(X) }, \, 
                    g \brb{ Q_{D_K,g}^{(K)}(X) }
                }
            }
            }_{\text{oracle-trajectory risk (OTR)}}.
    \end{equation}
    This tight (see \Cref{thm:main-no-free-lunch}) inequality follows directly from the definition of reasoning risk, the triangle inequality of $\ell$, and the almost-sure $(D_K,g)$-recoverability of $X$.
The TMR term measures the discrepancy, under the hypothesis $f$, between the answers to the final questions along the two CoT trajectories generated by $f$ and by the ground truth $g$.
The OTR term, on the other hand, measures the expected loss incurred by using $f$ to answer the final question along the CoT trajectory correctly generated by the ground truth $g$.
Thus, TMR and OTR capture the cost and benefit of CoT, respectively.
We elaborate on these two terms in the next sections.

\section{Benefit of CoT}
\label{sec:benefit-cot}

Recall that the classical notion of statistical risk of 
a hypothesis $f$ with respect to the ground truth $g$ under a distribution $\mu$ on $\cX$ and measured by the loss $\ell$ is defined by
\[
    R_{\mu,\ell}(f,g) 
\coloneqq 
    \underset{X\sim\mu}{\E} \Bsb{ \ell \brb{ f(X), \, g(X) } }.
\]
The OTR term is exactly
$R_{Q_{D_K,g}^{(K)}\#\nu,\ell}(f,g)$, namely the statistical risk of $f$ with respect to $g$ under the push-forward of $\nu$ by the $K$-th question $Q_{D_K,g}^{(K)}$.
Since the LLM is trained on a distribution, say $\mu$, that may differ from the testing distribution $\nu$, minimizing the OTR coincides with the objective of a domain adaptation problem~\citep{ben-david_analysis_2006}: $\mu$ is the source-domain distribution, the push-forward distribution $Q_{D_K,g}^{(K)}\#\nu$ is the target-domain distribution, and one aims to learn a hypothesis, often from a labeled source sample and an unlabeled target sample, that minimizes the target-domain risk.

This observation has a useful implication.
When the push-forward distribution $Q_{D_K,g}^{(K)}\#\nu$ coincides with the training distribution $\mu$, minimizing the OTR is exactly the usual supervised-learning objective under the training distribution.
On the other hand, when $Q_{D_K,g}^{(K)}\#\nu$ differs from $\mu$, a rich body of results in the domain adaptation literature can be used to study guarantees and hardness for minimizing OTR.
For example, existing domain-adaptation bounds imply inequalities of the form
\[
{\rm OTR} \le  R_{\mu,\ell}(f,g) + {\rm Disc} \brb{ \mu, Q_{D_K,g}^{(K)}\#\nu } + R^*,
\]
where ${\rm Disc}(\cdot, \cdot)$ measures some notion of discrepancy between two distributions and $R^*$ is a term independent of $f$.
The terms $R_{\mu,\ell}(f,g)$ and ${\rm Disc} \brb{ \mu, Q_{D_K,g}^{(K)}\#\nu }$ can often be upper-bounded using empirical estimates from data, together with concentration inequalities involving suitable complexity measures.
As illustrated in \Cref{sec:OTR_bounds}, the main implication is that when CoT prompting can push forward the testing distribution in a way that closely aligns it with the training distribution, so that ${\rm Disc} \brb{ \mu, Q_{D_K,g}^{(K)}\#\nu }$ is small, one obtains a guarantee of low OTR.

We dedicate the remainder of this paper to understanding the behavior of the TMR term, which captures the cost of CoT.

    \section{Cost of CoT}

We now turn to the TMR term, which captures the cost of CoT.
We first show that, without appropriate stability assumptions, the TMR cannot be controlled even when the OTR is zero and the hypothesis is uniformly close to the ground truth.
We then prove a tight upper bound under stability.

    \subsection{No free lunch without stability}
    \label{sec:freeLunch}
    
    In this section, we investigate the necessity of stability, as captured by the following definition.
    
    \begin{definition}
    Let $n\in\N$, $(\cX,\rho_\cX)$ and $(\cY,\rho_\cY)$ be metric spaces,\footnote{%
    When $\cY=[0,\infty)$, we always implicitly equip $\cY$ with the usual Euclidean metric.
    }
    and $\gamma,\gamma_1, \dots, \gamma_n \ge 0$.
    A function $h \colon  \cX^n \to \cY$ is \emph{$\gamma$-stable (with coordinate-wise Lipschitz constants $\gamma_1, \dots, \gamma_n$)} if $\sum_{i=1}^n \gamma_i \le \gamma$ and, for all $x_i,x_i' \in \cX$,
    $
        \rho_\cY \brb{
            h(x_1,\ldots,x_n),
            h(x_1',\ldots,x_n')
        }
        \le
        \sum_{i=1}^n \gamma_i \rho_\cX (x_i,x_i')
    $.
    \end{definition}

    We now present our no-free-lunch theorem, which shows that unless stability holds simultaneously
for $f$, $\ell$, and each step of $D_K$, the TMR can be arbitrarily large, even if the OTR is $0$, the hypothesis
$f$ is arbitrarily close to the ground truth $g$, and two out of the three stability assumptions hold in a
strong sense (with arbitrarily small Lipschitz constants).
The construction works even under a uniform boundedness restriction on the loss: for any prescribed finite bound, the TMR can be made to attain the largest loss value allowed by that restriction.
Note also that the theorem shows, in particular, that the TMR + OTR decomposition in \eqref{eq:canonical-decomposition} is tight.

    \begin{theorem}
    \label{thm:main-no-free-lunch}
    For any integer $K \ge 2$, any $M > 0$, any $\e > 0$, and any assumption 
    $\mathsf A \in \{1,2,3\}$ (see below),
    there exist a metric space $(\cX,\rho)$, a quasimetric loss function
    $\ell \colon \cX \times \cX \to [0,M]$,
    two answer maps $f,g \colon \cX \to \cX$ such that
    $\sup_{x \in \cX} \rho \brb{f(x),g(x)} \le \e$, a $K$-step chain rule $D_K$, and a
    distribution $\nu$ on $\cX$ supported on $(D_K,g)$-recoverable points such that:
    \begin{itemize}
        \item if $\mathsf A=1$, then $\ell$ and all steps of $D_K$ are $\e$-stable;
        \item if $\mathsf A=2$, then $f$ and all steps of $D_K$ are $\e$-stable;
        \item if $\mathsf A=3$, then $f$ and $\ell$ are $\e$-stable;
    \end{itemize}
    and such that \eqref{eq:canonical-decomposition} holds with equality and
    \[
        \underset{X \sim \nu}{\mathbb{E}}
        \bbsb{
            \ell\Brb{
                f\brb{ Q_{D_K,f}^{(K)}(X) }, \,
                f\brb{ Q_{D_K,g}^{(K)}(X) }
            }
        }
        = M,
        \ \
        \underset{X \sim \nu}{\mathbb{E}}
        \bbsb{
            \ell\Brb{
                f \brb{ Q_{D_K,g}^{(K)}(X) }, \,
                g \brb{ Q_{D_K,g}^{(K)}(X) }
            }
        }
        = 0.
    \]
    \end{theorem}
    
    \begin{proof}[Proof sketch.]
    The claim follows by applying \Cref{lem:nfl-no-model-lipschitz,lem:nfl-no-loss-lipschitz,lem:nfl-no-chain-lipschitz} in Appendix \ref{appe:proof-of-no-free-lunch}, according to whether the omitted condition is the stability of $f$, $\ell$, or $D_K$, respectively.
    At a high level, all three constructions make the oracle-trajectory risk vanish by ensuring that the oracle's trajectory reaches a point where $f$ and $g$ agree, while the learner's trajectory reaches a point whose hypothesis answer incurs loss $M$ relative to the oracle-trajectory answer.
    When the stability of $f$ is dropped, a non-Lipschitz spike in $f$ turns a small, stable trajectory separation into a large output separation.
    When the stability of $\ell$ is dropped, a non-Lipschitz loss assigns loss $M$ to an arbitrarily small nonzero output discrepancy.
    When the stability of $D_K$ is dropped, a discontinuous step of the chain rule amplifies a tiny intermediate-answer discrepancy into a large trajectory mismatch.
    Thus, in each case, the TMR is $M$, the OTR is zero, and \eqref{eq:canonical-decomposition} holds with equality.
    \end{proof}

    This no-free-lunch theorem shows that the theoretical obstruction is not the OTR term or the \emph{accuracy} of the hypothesis $f$, but rather the \emph{instability} of $f$, the loss $\ell$, or any step of the chain rule $D_K$, which precludes bounding the TMR term.
    In the following section, we investigate what changes when, instead, these stability conditions hold.

    \subsection{A tight bound on the Trajectory-Mismatch Risk under stability}
    \label{sec:strong-stability}

    In this section, we prove a tight bound on the Trajectory-Mismatch Risk (TMR) under the assumption that the loss $\ell$ is $\lambda$-stable, the hypothesis $f$ is $\phi$-stable, and each step of the chain rule $D_K$ is $\delta$-stable.
    The key quantity in the bound is the following amplification factor.
    
    \begin{definition}
    \label{def:amplification-factor}
    For any integer $K\ge 2$, any $\phi\ge 0$, and any $\delta\ge 0$, define the \emph{amplification factor}
    \[
        \alpha_K(\phi,\delta)
    \coloneqq
        \begin{cases}
        \displaystyle
        \frac{1-(\phi\delta)^{K-1}}{1-\phi\delta},
        & \text{if } (\phi\ge 1 \text{ or } \delta\le 1),\ \text{and } \phi\delta\neq 1,
        \\
        \displaystyle
        K-1,
        & \text{if } (\phi\ge 1 \text{ or } \delta\le 1),\ \text{and } \phi\delta=1,
        \\
        \displaystyle
        \delta^{K-2-m_K(\phi,\delta)}
        \frac{1-(\phi\delta)^{m_K(\phi,\delta)+1}}{1-\phi\delta},
        & \text{if } \phi<1,\ \delta>1,\ \text{and } \phi\delta\neq 1,
        \\
        \displaystyle
        \delta^{K-2-n_K(\phi,\delta)} \brb{ n_K(\phi,\delta)+1 },
        & \text{if } \phi<1,\ \delta>1,\ \text{and } \phi\delta=1,
        \end{cases}
    \]
    where
    \[
        m_K(\phi,\delta)
    \coloneqq
        \begin{cases}
        0, & \text{if } \phi=0,
        \\[0.4em]
        \displaystyle
        \min\bbrb{
            K-2,
            \bbfl{
            \log_{\phi\delta}\Brb{\frac{\delta-1}{\delta(1-\phi)}}
            }
        },
        & \text{if } \phi>0,
        \end{cases}
    \qquad
        n_K(\phi,\delta)
    \coloneqq
        \min\bbrb{
            K-2,\
            \bbfl{ \frac{1}{\delta-1} }
        }.
    \]
    \end{definition}
    
    The amplification factor $\alpha_K(\phi,\delta)$ gives an \emph{exact} theoretical picture of how CoT errors evolve with the number of reasoning steps.
In the first regime of the definition, $\alpha_K(\phi,\delta)$ is simply a geometric sum with ratio $\phi\delta$.
In that regime, CoT errors remain controlled when $\phi\delta<1$, accumulate linearly when $\phi\delta=1$, and snowball exponentially when $\phi\delta>1$.
This provides a formal explanation for phenomena observed in CoT reasoning, including stable reasoning, gradual drift, and snowballing error accumulation~\citep{gan_rethinking_2025,bachmann_pitfalls_2025,zhu_advchain_2025}.
Related theoretical and empirical work has also emphasized that the quality and stability of intermediate reasoning steps can affect generalization~\citep{hu_unveiling_2024,tutunov_why_2024,cui_theoretical_2024,chen_reasoning_2025,madaan_text_2022}.
    
    The mixed regime $\phi<1<\delta$ shows that the story is not governed only by the product $\phi\delta$.
    Even if the hypothesis itself contracts question perturbations, the chain rule may amplify previous question mismatches strongly enough to create exponential growth.
    The maximization problem determining $\alpha_K$ (see \Cref{lem:max-representation-amplification-factor}, in Appendix \ref{appe:proofOfPositiveResult}) reflects the worst-case instance of this problem, built by injecting discrepancy through answer-dependent steps, then amplifying it through question-dependent steps.
    Thus, the bound identifies two distinct sources of CoT instability: instability of the hypothesis answer map through $\phi$, and instability of the reasoning rule through $\delta$. 
    
    With the amplification factor determined exactly, we can now present our main positive result: an upper bound on the TMR under stability of $\ell$, $f$, and $D_K$.
    \begin{theorem}
    \label{thm:stability-TMR}
    Fix any integer $K \ge 2$. 
    Let $\lambda,\phi,\delta\ge 0$, 
    $(\cX,\rho)$ be a metric space, 
    $\ell\colon\cX\times\cX\to[0,\infty)$ be a quasimetric loss function,
    $f,g\colon\cX\to\cX$ be two answer maps such that
    $
        d(f,g)
    \coloneqq
        \sup_{x\in\cX}\rho\brb{f(x),g(x)}<\infty
    $,
    $D_K$ be a $K$-step chain rule,
    and $\nu$ a distribution supported on $(D_K,g)$-recoverable points.\footnote{The assumption that $\nu$ is supported on $(D_K,g)$-recoverable points is actually not needed for the validity of \Cref{thm:stability-TMR}. 
    The reason why we still list it as an assumption is that \Cref{thm:stability-TMR} is meant to be used in conjunction with our canonical decomposition \eqref{eq:canonical-decomposition} to control the reasoning risk (see \Cref{cor:stability-reasoning-risk}), and Inequality \eqref{eq:canonical-decomposition} does need this assumption.}
    If $\ell$ is $\lambda$-stable,
    $f$ is $\phi$-stable,
    and the steps of $D_K$ are $\delta$-stable,
    then the Trajectory-Mismatch Risk (TMR) satisfies
    \[
        \underset{X \sim \nu}{\E}
            \bbsb{
                \ell\Brb{
                    f\brb{ Q_{D_K,f}^{(K)}(X) }, \,
                    f\brb{ Q_{D_K,g}^{(K)}(X) }
                }
            }
        \le
        \frac{\lambda\phi\delta}{2}\,
        \alpha_K(\phi,\delta)\,
        d(f,g).
    \]
    \end{theorem}
    For a full proof of this result, see Appendix \ref{appe:proofOfPositiveResult}.
    \begin{proof}[Proof sketch.]
    The proof first tracks how discrepancies propagate through the two CoT trajectories.
    Let $\Delta_k(x)$ be the distance between the $k$-th questions generated by $f$ and $g$, and let $\Gamma_k(x)$ be the corresponding distance between the $k$-th answers.
    The stability of $f$ gives
    $
        \Gamma_k(x)\le \phi\Delta_k(x)+d(f,g),
    $
    while the $\delta$-stability of the chain-rule steps implies that each new question mismatch is controlled by the largest previous question or answer mismatch.
    The case $d(f,g)=0$ is trivial.
    When $d(f,g)>0$, after normalizing by $d(f,g)$, this leads to a worst-case discrete dynamical system with two possible update maps:
    $
    z\mapsto 
        T_{\mathsf Q}(z)
    \coloneqq
        \delta z
    $ 
    and
    $
    z\mapsto 
        T_{\mathsf A}(z)
    \coloneqq
        \delta(\phi z+1).
    $
    The first map corresponds to a chain-rule step depending on a previous question coordinate, while the second corresponds to a step depending on a previous answer coordinate, which introduces both the propagated error $\phi z$ and a fresh unit discrepancy between $f$ and $g$.
    The main technical step is then to solve exactly the resulting worst-case amplification problem over all words in the alphabet $\{\mathsf Q,\mathsf A\}$ of length at most $K-1$.
    This gives
    $
        \Delta_K(x)
        \le
        \delta \, \alpha_K(\phi,\delta) \, d(f,g),
    $
    where the explicit form of $\alpha_K$ is obtained by the max representation in Appendix \ref{appe:proofOfPositiveResult}, \Cref{lem:max-representation-amplification-factor}.
    Finally, using the fact that $\ell$ is $\lambda$-stable and vanishes on the diagonal, one obtains
$\ell(u,v)\le \lambda \rho(u,v)/2$ for all $u,v\in\cX$.
    Putting all these estimates together yields the result.
    \end{proof}
Combining our decomposition \eqref{eq:canonical-decomposition} with \Cref{thm:stability-TMR} immediately yields control of the reasoning risk.
\begin{corollary}
\label{cor:stability-reasoning-risk}
Under the assumptions of \Cref{thm:stability-TMR}, the reasoning risk satisfies
\[
    R^{D_K}_{\nu,\ell}(f,g)
    \le
    \frac{\lambda\phi\delta}{2}\,
    \alpha_K(\phi,\delta)\,
    d(f,g)
    +
    \E_{Y \sim Q_{D_K,g}^{(K)}\#\nu} \Bsb{ \ell \brb{ f(Y), \, g(Y) } }.
\]
\end{corollary}
We conclude this section by showing that the previous bound is tight \emph{in all regimes}.
\begin{theorem}
\label{thm:stability-reasoning-risk-tight}
Fix any integer $K \ge 2$.
Let $\lambda,\phi,\delta\ge 0$.
Then, there exist a metric space $(\cX,\rho)$, 
a quasimetric loss function
$\ell\colon \cX\times\cX\to[0,\infty)$, 
two answer maps $f,g\colon \cX\to\cX$ such that 
$
    d(f,g)
\coloneqq
    \sup_{x\in\cX}\rho\brb{f(x),g(x)}
<
    \infty
$, 
a $K$-step chain rule $D_K$, and a
distribution $\nu$ supported on $(D_K,g)$-recoverable points such that 
$\ell$ is $\lambda$-stable,
$f$ is $\phi$-stable,
the steps of $D_K$ are $\delta$-stable, 
the oracle-trajectory risk is zero,
and the bound in \Cref{cor:stability-reasoning-risk} is attained.
\end{theorem}
For a full proof of this result, see Appendix \ref{appe:proofOfOptimalityOfPositiveResult}.
\begin{proof}[Proof sketch.]
The construction mirrors the worst-case recursion in the proof of \Cref{thm:stability-TMR}.
Choose a word $w\in\{\mathsf A,\mathsf Q\}^{K-1}$ attaining the maximum in the $\max$ representation of $\alpha_K(\phi,\delta)$ (see \Cref{lem:max-representation-amplification-factor} in Appendix \ref{appe:proofOfPositiveResult}), and let the chain rule follow this word: an $\mathsf A$-step depends only on the previous answer, while a $\mathsf Q$-step depends only on the previous question.
We take $\cX \coloneqq \R$, $z\mapsto f(z) \coloneqq \phi z$, and define $g$ so that it differs from $f$ by exactly one unit on the intermediate oracle questions $s_1,\ldots,s_{K-1}$, but agrees with $f$ at the terminal oracle question $s_K$.
The chain rule is calibrated so that the oracle trajectory follows exactly $s_1,\ldots,s_K$, making the OTR term equal to zero, while the learner trajectory follows $s_k+E_{k-1}$, where
$
    E_r \coloneqq T_{w_r}\circ\cdots\circ T_{w_1}(0)
$, 
$
    z \mapsto
    T_{\mathsf Q}(z) \coloneqq \delta z
$,
and 
$
    z \mapsto
    T_{\mathsf A}(z) \coloneqq \delta(\phi z+1).
$
By the choice of $w$, this yields
$
    E_{K-1}
    =
    \delta\alpha_K(\phi,\delta)
$.
Finally, with the quasimetric loss $(u,v) \mapsto \ell(u,v) \coloneqq \lambda\labs{u-v}/2$, the final learner answer is separated from the oracle answer by exactly $\phi E_{K-1}$, and therefore the reasoning risk is
$
    \frac{\lambda\phi\delta}{2}\alpha_K(\phi,\delta)
$.
Since $d(f,g) = 1 < \infty$ and the OTR term is zero, the upper bound of \Cref{cor:stability-reasoning-risk} is attained.
\end{proof}

\section{Additional related work}
\label{s:related}

\paragraph{Chain of Thought.}
Since the introduction of Chain of Thought (CoT)~\citep{nye_show_2021, wei_chain--thought_2023, kojima_large_2023}, a substantial body of work has demonstrated its effectiveness in improving reasoning performance across a wide range of tasks~\citep{bao_how_2025, sprague_cot_2025, stechly_chain_2025, liu_can_2024, chen_towards_2025}, while also showing that CoT is highly sensitive to prompt design~\citep{madaan_text_2022}.
Much work has aimed to improve the performance, reliability, and efficiency of CoT~\citep{yao_tree_2023, snell_scaling_2024, wang_chain--thought_2024, ma_cot-valve_2025, zhou_inform_2023, hao_training_2025, zhu_scaling_2025, geiping_scaling_2025, saunshi_reasoning_2025, chen_theoretical_2025, cheng_compressed_2024, besta_graph_2024, cheng_integrative_2025, chen_universal_2024, malon_self-consistent_2024, besta_demystifying_2025, zhang_why_2025, liu_logic--thought_2025, chia_contrastive_2023, zhang_chain_2024, deepseek-ai_deepseek-r1_2025, ning_not_2025, feng_alphazero-like_2024, yang_markov_2025, aghajohari2026the, yuan_breaking_2026} and to develop evaluation frameworks~\citep{korbak_chain_2025, chen_unlocking_2024, chen_towards_2025}.
Known failure modes include overthinking, where unnecessarily long reasoning chains degrade performance~\citep{wu_when_2025, chen_not_2025, su_between_2025}; inconsistency, where different reasoning paths lead to conflicting conclusions~\citep{chen_reasoning_2025, madaan_text_2022, turpin_language_2023}; and snowball errors, where early mistakes propagate and amplify along the reasoning chain~\citep{gan_rethinking_2025, bachmann_pitfalls_2025, zhu_advchain_2025}.
CoT-based methods can also lack robustness under distribution shift~\citep{zhao_is_2026, mirzadeh_gsm-symbolic_2025, tang_large_2023, mondorf_beyond_2024}.
Theoretically, CoT and LLM reasoning have been studied through Transformer expressivity, optimization, and generalization~\citep{vaswani_attention_2023, perez_attention_2021, feng_towards_2023, li_chain_2024, huang_transformers_2025-1, merrill_expressive_2024, chen_theoretical_2025, barcelo_ehrenfeucht-haussler_2025, amiri_lower_2025, lee_how_2025, wen_sparse_2024, kim_transformers_2025, sanford_representational_2023, peng_limitations_2024, wen_rnns_2024, hahn_theoretical_2020, zhu_reasoning_2025, gozeten_continuous_2026, hao_training_2025, cui_theoretical_2024, li_training_2024, yang_multi-head_2025, huang_transformers_2025}, as well as from complexity-theoretic~\citep{malach_auto-regressive_2024}, learning-theoretic~\citep{wies_sub-task_2023, joshi_theory_2025, hu_unveiling_2024, tutunov_why_2024}, information-theoretic~\citep{prystawski_why_2023, ton_understanding_2025, qian_demystifying_2025, yao_compositional_2026, yong_think_2025, altabaa_cot_2025, gan_rethinking_2025}, and limitation-focused perspectives~\citep{bachmann_pitfalls_2025}.

\paragraph{Domain adaptation.}
In domain adaptation~\citep{ben-david_analysis_2006}, a learner uses knowledge from a source domain to perform well on a target domain, where the two domains are characterized by different distributions.
A rich literature studies target-risk guarantees~\citep{shen_wasserstein_2018, zhang_bridging_2019, acuna_f-domain_2021, wang_f-divergence_2024} and hardness results~\citep{ben-david_impossibility_2010, ben-david_hardness_2012, ben-david_domain_2014, redko_analysis_2019, hanneke_value_2019, dong_hardness_2026}.
These results use discrepancy measures between source and target domains, such as the $H\Delta H$-divergence~\citep{ben-david_analysis_2006, blitzer_learning_2007, ben-david_theory_2010}, transfer exponents~\citep{hanneke_value_2019}, and $f$-divergence-based discrepancies~\citep{wang_f-divergence_2024}, and have guided the design of domain-adaptation algorithms~\citep{roark_supervised_2003, jiang_instance_2007, ganin_domain-adversarial_2016, shen_wasserstein_2018, long_deep_2017, saito_maximum_2018}.

\section{Discussion and limitations}

Our formulation uses a fixed number $K$ of CoT steps.
This is a useful simplification, but it does not fully capture how LLMs operate in practice, where the number of reasoning steps may depend on the input and is only determined during generation.
One way to accommodate variable-length CoT is to define a chain rule as an infinite sequence of functions $D^{(1)},D^{(2)},\ldots$, with $D^{(k)}\colon \cX^{2k-1}\to\cX$, together with a stopping rule.
For example, when the chain rule interacts with the ground-truth answer map $g$, one may define a stopping rule $K_{D,g}\colon \cX\to\N$ and, for each input $x$, set $K_x \coloneqq K_{D,g}(x)$ and $D_{K_x} \coloneqq (D^{(1)},\ldots,D^{(K_x)})$.
The results in this paper can then be viewed as applying to each fixed value of the selected reasoning length.
We have chosen the fixed-$K$ formulation because it keeps the notation and concepts simpler, allowing us to focus on the main cost-benefit decomposition and stability phenomena.

A limitation of this work is that the current formulation treats chain rules and answer maps as deterministic, whereas practical LLMs often use stochastic decoding.
This choice allows us to put the spotlight on the main phenomenon of interest: how errors propagate along CoT trajectories and how stability controls this propagation.
We expect stochastic chain rules and stochastic answer maps to require additional notation and bookkeeping, but not to alter the basic cost-benefit picture developed here.
Another scope restriction is that we assume the relevant chain rule is available at test time.
This captures settings in which the chain rule and the answer map can be specified or learned separately, but it does not model situations in which learning the chain rule and learning the answer map are coupled.
Understanding this coupled learning problem is an important direction for future work.
Finally, our analysis focuses on test distributions supported on $(D_K,g)$-recoverable points.
This assumption isolates the behavior of CoT on inputs whose answers can in principle be recovered by the oracle-guided chain rule.
If the test distribution also assigns mass to non-recoverable points, one can partition the distribution into its recoverable and non-recoverable parts: the former is controlled by the TMR/OTR analysis developed here, while the latter requires an additional oracle-mismatch term to be controlled separately.
We leave this more general, but less transparent, formulation to future work.

    \newpage
    \bibliographystyle{abbrvnat}
    \bibliography{references}

    \clearpage
    \appendix

    \section{Proofs of the main results}
    \label[appendix]{appe:proofs}

    In this section, we present the full proofs of all our main results.
    
    \subsection{Proof for Theorem \ref{thm:main-no-free-lunch}}
    \label[appendix]{appe:proof-of-no-free-lunch}
    
    We split the proof of the theorem into three parts, which imply the result in the three cases where the stability of $f$, $\ell$, or the steps of $D_K$ is removed.
    
    If the stability of the hypothesis $f$ is removed, we obtain the following result.
    
    \begin{lemma}[No free lunch without stability of the hypothesis]
    \label{lem:nfl-no-model-lipschitz}
    For any integer $K \ge 2$, any $M > 0$, and any $\e > 0$,
    there exist a metric space $(\cX,\rho)$, a quasimetric loss function
    $\ell \colon \cX \times \cX \to [0,M]$ that is
    $\e$-stable, two functions $f,g \colon \cX \to \cX$ such that
    $
        \sup_{x \in \cX} \rho \brb{ f(x),g(x) } \le \e
    $,
    a $K$-step chain rule $D_K$ whose steps are $\e$-stable, and a distribution $\nu$ on $\cX$ whose support
    consists of $(D_K,g)$-recoverable points such that \eqref{eq:canonical-decomposition} holds
    with equality and
    \[
        \underset{X \sim \nu}{\mathbb{E}}
        \bbsb{
            \ell\Brb{
                f\brb{ Q_{D_K,f}^{(K)}(X) }, \,
                f\brb{ Q_{D_K,g}^{(K)}(X) }
            }
        }
        = M,
        \quad
        \underset{X \sim \nu}{\mathbb{E}}
        \bbsb{
            \ell\Brb{
                f \brb{ Q_{D_K,g}^{(K)}(X) }, \, 
                g \brb{ Q_{D_K,g}^{(K)}(X) }
            }
        }
        = 0.
    \]
    \end{lemma}
    
    \begin{proof}
    Fix any integer $K \ge 2$, any $M > 0$, and any $\e > 0$.
    Let
    \[
        L \coloneqq \min \bcb{1,\nicefrac{\e}{2}},
        \qquad
        \eta \coloneqq \min \bcb{\e,\nicefrac{1}{2L^{K-1}}},
        \qquad
        x_\star \coloneqq L^{K-1}\eta,
        \qquad
        Y \coloneqq \nicefrac{2M}{\e}.
    \]
    By definition, $x_\star \in (0,1/2]$.
    Let
    \[
        B \coloneqq \max \bcb{1,\e,Y}
        \qquad
        \text{and}
        \qquad
        \cX \coloneqq [0,B].
    \]
    For any $x,y \in \cX$, let
    \[
        \rho(x,y) \coloneqq \labs{x-y}
        \qquad
        \text{and}
        \qquad
        \ell(x,y) \coloneqq \min \lcb{\frac{\e}{2} \labs{x-y} , \, M}.
    \]
    Since $\rho$ is a metric, the function
\[
    (x,y)\mapsto \min\left\{ \frac{\e}{2}\rho(x,y),\, M \right\}
\]
is a quasimetric bounded by $M$.
Indeed, for all $x,y,z\in\cX$,
\[
    \ell(x,y)
    \le
    \min\left\{ \frac{\e}{2}\rho(x,z)+\frac{\e}{2}\rho(z,y),\, M \right\}
    \le
    \ell(x,z)+\ell(z,y).
\]
Moreover, since $t\mapsto \min\{t,M\}$ is $1$-Lipschitz, for all $x,x',y,y'\in\cX$,
\[
    |\ell(x,y)-\ell(x',y')|
    \le
    \frac{\e}{2}\rho(x,x')
    +
    \frac{\e}{2}\rho(y,y').
\]
Thus, $\ell$ is $\e$-stable.
    
    Now define $f,g \colon \cX \to \cX$, for all $x \in \cX$, by
    \[
        f(x)
        \coloneqq
        \begin{cases}
        Y, & \text{if } x=x_\star,\\
        \eta, & \text{if } x=1,\\
        0, & \text{otherwise,}
        \end{cases}
    \]
    and
    \[
        g(x)
        \coloneqq
        \begin{cases}
        Y, & \text{if } x=x_\star,\\
        0, & \text{otherwise.}
        \end{cases}
    \]
    Since $x_\star \neq 1$, the two functions differ only at $x=1$.
    Therefore,
    \[
        \sup_{x \in \cX} \rho\brb{f(x),g(x)}
        =
        \eta
        \le
        \e.
    \]
    Let the testing distribution be the Dirac measure at $0$,
    \[
        \nu \coloneqq \delta_0.
    \]
    We now define the chain rule.
    Set, for all $x \in \cX$,
    \[
        D_K^{(1)}(x) \coloneqq 1.
    \]
    For the second step, define, for any $x,q_1,a_1 \in \cX$,
    \[
        D_K^{(2)}(x,q_1,a_1)
        \coloneqq
        L a_1.
    \]
    For every $k \in \{3,\dots,K\}$, define, for any
    $x,q_1,a_1,\ldots,q_{k-1},a_{k-1} \in \cX$,
    \[
        D_K^{(k)}(x,q_1,a_1,\dots,q_{k-1},a_{k-1})
        \coloneqq
        L q_{k-1}.
    \]
    Since $L \le 1$, all these maps take values in $\cX$.
    
    The steps of the chain rule are $\e$-stable.
    Indeed, $D_K^{(1)}$ is $0$-stable.
    Moreover, $D_K^{(2)}$ is $L$-stable, since it is $L$-Lipschitz in the coordinate $a_1$ and does not depend on any other coordinate.
    For every $k \in \{3,\dots,K\}$, $D_K^{(k)}$ is $L$-stable, since it is $L$-Lipschitz in the coordinate $q_{k-1}$ and does not depend on any other coordinate.
    Since $L \le \e$, all steps of $D_K$ are $\e$-stable.
    
    We now compute the oracle's trajectory for the initial prompt $x=0$.
    Since
    \begin{align*}
    Q^{(1)}_{D_K,g}(0) &= 1,
    \\
    A^{(1)}_{D_K,g}(0) &= g(1)=0,
    \end{align*}
    we get
    \begin{align*}
    Q^{(2)}_{D_K,g}(0) &= L A^{(1)}_{D_K,g}(0)=0,
    \\
    A^{(2)}_{D_K,g}(0) &= g(0)=0.
    \end{align*}
    Moreover, for all later steps $k \in \{3,\dots,K\}$, since the chain rule maps
    $q_{k-1}$ to $Lq_{k-1}$, we get
    \begin{align*}
    Q^{(k)}_{D_K,g}(0) &= 0,
    \\
    A^{(k)}_{D_K,g}(0) &= g(0)=0.
    \end{align*}
    In particular,
    \[
        A^{(K)}_{D_K,g}(0)=0=g(0),
    \]
    therefore $0$ is $(D_K,g)$-recoverable, and the support of $\nu$ is contained in the
    $(D_K,g)$-recoverable set.
    
    We now compute the learner's trajectory for the initial prompt $x=0$.
    Since
    \begin{align*}
    Q^{(1)}_{D_K,f}(0) &= 1,
    \\
    A^{(1)}_{D_K,f}(0) &= f(1)=\eta,
    \end{align*}
    the second question is
    \[
        Q^{(2)}_{D_K,f}(0)
        =
        L\eta.
    \]
    For every $k \in \{3,\dots,K\}$, the chain rule then gives
    \[
        Q^{(k)}_{D_K,f}(0)
        =
        L Q^{(k-1)}_{D_K,f}(0).
    \]
    Hence, by induction,
    \[
        Q^{(k)}_{D_K,f}(0)
        =
        L^{k-1}\eta
        \quad
        \text{for all } k \in \{2,\dots,K\}.
    \]
    In particular,
    \[
        Q^{(K)}_{D_K,f}(0)
        =
        L^{K-1}\eta
        =
        x_\star.
    \]
    Therefore,
    \[
        A^{(K)}_{D_K,f}(0)
        =
        f(x_\star)
        =
        Y.
    \]
    
    Consequently, the reasoning risk is
    \[
        R^{D_K}_{\nu,\ell}(f,g)
        =
        \underset{X \sim \nu}{\mathbb{E}}
        \bbsb{
            \ell\Brb{
                A^{(K)}_{D_K,f}(X), \
                g(X)
            }
        }
    =
        \ell\brb{Y,g(0)}
    =
        \ell\brb{Y,0}
    =
        \min\left\{\frac{\e}{2}Y,M\right\}
    =
        M.
    \]
    The trajectory-mismatch risk is
    \[
        \underset{X \sim \nu}{\mathbb{E}}
        \bbsb{
            \ell\Brb{
                f \brb{ Q^{(K)}_{D_K,f}(X) },\ 
                f \brb{ Q^{(K)}_{D_K,g}(X) }
            }
        }
    =
        \ell\brb{f(x_\star),f(0)}
    =
        \ell\brb{Y,0}
    =
        M,
    \]
    and the oracle-trajectory risk is
    \[
    \underset{X \sim \nu}{\mathbb{E}}
        \bbsb{
            \ell\Brb{
            f\brb{ Q^{(K)}_{D_K,g}(X) }, \
            g\brb{ Q^{(K)}_{D_K,g}(X) }
        }
    }
    =
        \ell\brb{f(0),g(0)}
    =
        \ell\brb{0,0}
    =
        0.
    \]
    This proves, in particular, that the right-hand side of the decomposition is $M+0=M$, which equals the
    reasoning risk, hence \eqref{eq:canonical-decomposition} holds with equality.
    \end{proof}

    If the stability of the loss $\ell$ is removed, we obtain the following result.
    
    \begin{lemma}[No free lunch without stability of the loss]
    \label{lem:nfl-no-loss-lipschitz}
    For any integer $K \ge 2$, any $M > 0$, and any $\e > 0$,
    there exist a metric space $(\cX,\rho)$, a quasimetric loss function
    $\ell \colon \cX \times \cX \to [0,M]$,
    two $\e$-stable functions $f,g \colon \cX \to \cX$ such that
    $
        \sup_{x \in \cX} \rho \brb{ f(x),g(x) } \le \e
    $,
    a $K$-step chain rule $D_K$ whose steps are $\e$-stable,
    and a distribution $\nu$ on $\cX$ whose support
    consists of $(D_K,g)$-recoverable points such that  \eqref{eq:canonical-decomposition} holds with equality and
    \[
        \underset{X \sim \nu}{\mathbb{E}}
        \bbsb{
            \ell\Brb{
                f\brb{ Q_{D_K,f}^{(K)}(X) }, \,
                f\brb{ Q_{D_K,g}^{(K)}(X) }
            }
        }
        = M,
        \quad
        \underset{X \sim \nu}{\mathbb{E}}
        \bbsb{
            \ell\Brb{
                f \brb{ Q_{D_K,g}^{(K)}(X) }, \, 
                g \brb{ Q_{D_K,g}^{(K)}(X) }
            }
        }
        = 0.
    \]
    \end{lemma}
    
    \begin{proof}
    Fix any integer $K \ge 2$, any $M > 0$, and any $\e > 0$.
    Let
    \[
        L \coloneqq \min \bcb{\e,1},
    \]
    and let $\cX \coloneqq [0,1]$.
    For any $x,y \in \cX$, let
    \[
        \rho(x,y) \coloneqq \labs{x-y}.
    \]
    Define the loss function $\ell \colon \cX \times \cX \to [0,M]$, for all $x,y \in \cX$, by
    \[
        \ell(x,y)
        \coloneqq
        \begin{cases}
        0, & \text{if } x=y,\\
        M, & \text{if } x\neq y.
        \end{cases}
    \]
    Then $\ell$ satisfies the triangle inequality and for all $x\in \cX$, $\ell(x,x)=0$. Hence $\ell$ is a quasimetric.
    Indeed, for any $x,y,z \in \cX$, if $x=y$, then $\ell(x,y)=0$ and the claim is immediate.
    If $x\neq y$, then at least one of $x\neq z$ or $z\neq y$ must hold, and therefore
    \[
        \ell(x,z)+\ell(z,y) \ge M = \ell(x,y).
    \]
    Now define $f,g \colon \cX \to \cX$, for all $x \in \cX$, by
    \[
        f(x) \coloneqq Lx,
        \qquad
        g(x) \coloneqq 0.
    \]
    Then $f$ is $L$-stable and $g$ is $0$-stable.
    Since $L \le \e$, both functions are $\e$-stable.
    Moreover,
    \[
        \sup_{x \in \cX} \rho\brb{f(x),g(x)}
        =
        \sup_{x \in [0,1]} Lx
        =
        L
        \le
        \e.
    \]
    Let the testing distribution be the Dirac measure at $1$,
    \[
        \nu \coloneqq \delta_1.
    \]
    
    We now define the chain rule.
    Set, for all $x \in \cX$,
    \[
        D_K^{(1)}(x) \coloneqq 1.
    \]
    For every $k \in \{2,\dots,K\}$, define, for any $x,q_1,a_1,\ldots,q_{k-1},a_{k-1} \in \cX$,
    \[
        D_K^{(k)}(x,q_1,a_1,\dots,q_{k-1},a_{k-1})
        \coloneqq
        L a_{k-1}.
    \]
    Since $a_{k-1} \in [0,1]$ and $L \le 1$, the right-hand side belongs to $\cX$.
    
    The steps of the chain rule are $\e$-stable.
    Indeed, $D_K^{(1)}$ is $0$-stable.
    Moreover, for every $k \in \{2,\dots,K\}$, the step $D_K^{(k)}$ is $L$-stable, since it is $L$-Lipschitz in the coordinate $a_{k-1}$ and does not depend on any other coordinate.
    Since $L\le \e$, all steps of $D_K$ are $\e$-stable.
    
    We now compute the oracle's trajectory for the initial prompt $x=1$.
    Since
    \begin{align*}
    Q^{(1)}_{D_K,g}(1) &= 1,
    \\
    A^{(1)}_{D_K,g}(1) &= g(1)=0,
    \end{align*}
    then, for every $k \in \{2,\dots,K\}$, we get
    \begin{align*}
    Q^{(k)}_{D_K,g}(1) &= L A^{(k-1)}_{D_K,g}(1)=0,
    \\
    A^{(k)}_{D_K,g}(1) &= g(0)=0 .
    \end{align*}
    In particular,
    \[
        A^{(K)}_{D_K,g}(1)=0=g(1),
    \]
    therefore $1$ is $(D_K,g)$-recoverable, and the support of $\nu$ is contained in the
    $(D_K,g)$-recoverable set.
    
    We now compute the learner's trajectory for the initial prompt $x=1$.
    Since
    \begin{align*}
    Q^{(1)}_{D_K,f}(1) &= 1,
    \\
    A^{(1)}_{D_K,f}(1) &= f(1)=L,
    \end{align*}
    and since, for every $k \in \{2,\dots,K\}$,
    \[
        Q^{(k)}_{D_K,f}(1)=L A^{(k-1)}_{D_K,f}(1),
        \qquad
        A^{(k)}_{D_K,f}(1)=f\brb{Q^{(k)}_{D_K,f}(1)}
        =
        L Q^{(k)}_{D_K,f}(1),
    \]
    we have, by induction,
    \[
        A^{(k)}_{D_K,f}(1)=L^{2k-1}
        \quad
        \text{for all } k \in [K].
    \]
    In particular,
    \[
        A^{(K)}_{D_K,f}(1)=L^{2K-1}>0.
    \]
    
    Therefore, the reasoning risk is
    \[
        R^{D_K}_{\nu,\ell}(f,g)
        =
        \underset{X \sim \nu}{\mathbb{E}}
        \bbsb{
            \ell\Brb{
                A^{(K)}_{D_K,f}(X), \
                g(X)
            }
        }
    =
        \ell \brb{ L^{2K-1}, \ g(1) }
    =
        \ell \brb{ L^{2K-1},0 }
    =
        M,
    \]
    where the last equality follows from $L^{2K-1}>0$.
    The trajectory-mismatch risk is
    \[
        \underset{X \sim \nu}{\mathbb{E}}
        \bbsb{
            \ell\Brb{
                f \brb{ Q^{(K)}_{D_K,f}(X) },\ 
                f \brb{ Q^{(K)}_{D_K,g}(X) }
            }
        }
    =
        \ell \brb{ A^{(K)}_{D_K,f}(1), \ A^{(K)}_{D_K,g}(1) }
    =
        \ell \brb{ L^{2K-1},0 }
    =
        M,
    \]
    and the oracle-trajectory risk is
    \[
    \underset{X \sim \nu}{\mathbb{E}}
        \bbsb{
            \ell\Brb{
            f\brb{ Q^{(K)}_{D_K,g}(X) }, \
            g\brb{ Q^{(K)}_{D_K,g}(X) }
        }
    }
    =
        \ell \brb{ f(0),\ g(0) }
    =
        \ell \brb{ 0,0 }
    =
        0 .
    \]
    This proves, in particular, that the right-hand side of the decomposition is $M+0=M$, which equals the
    reasoning risk, hence \eqref{eq:canonical-decomposition} holds with equality.
    \end{proof}

    If the stability of the steps of the chain rule $D_K$ is removed, we obtain the following result.

    \begin{lemma}
    [No free lunch without stability of the chain rule]
    \label{lem:nfl-no-chain-lipschitz}
    For any integer $K \ge 2$, any $M > 0$, and any $\e > 0$,
    there exist a metric space $(\cX,\rho)$, a quasimetric loss function
    $\ell \colon \cX \times \cX \to [0,M]$ that is $\e$-stable,
    two $\e$-stable functions $f,g \colon \cX \to \cX$ such that $\sup_{x \in \cX} \rho \brb{ f(x),g(x) } \le \e$, a $K$-step chain rule $D_K$, and a
    distribution $\nu$ on $\cX$ whose support consists of $(D_K,g)$-recoverable points such that \eqref{eq:canonical-decomposition} holds with equality and
    \[
        \underset{X \sim \nu}{\mathbb{E}}
        \bbsb{
            \ell\Brb{
                f\brb{ Q_{D_K,f}^{(K)}(X) }, \,
                f\brb{ Q_{D_K,g}^{(K)}(X) }
            }
        }
        = M,
        \quad
        \underset{X \sim \nu}{\mathbb{E}}
        \bbsb{
            \ell\Brb{
                f \brb{ Q_{D_K,g}^{(K)}(X) }, \, 
                g \brb{ Q_{D_K,g}^{(K)}(X) }
            }
        }
        = 0.
    \]
    \end{lemma}
    
    \begin{proof}
    Fix any integer $K \ge 2$, any $M > 0$, and any $\e > 0$.
    Let $B \coloneqq \max \bcb{ 1 + \nicefrac{2M}{\e^2}, \nicefrac{2M}{\e}, \e  }$ and $\cX \coloneqq [0,B]$.
    For any $x,y \in \cX$, let $\rho(x,y) \coloneqq \labs{ x - y }$ and 
    \[
        \ell(x,y) 
    \coloneqq 
        \min \lcb{ \frac{\e}{2} \rho(x,y), \, M }.
    \]
    Since $\rho$ is a metric, the function
\[
    (x,y)\mapsto \min\left\{ \frac{\e}{2}\rho(x,y),\, M \right\}
\]
is a quasimetric bounded by $M$.
Indeed, for all $x,y,z\in\cX$,
\[
    \ell(x,y)
    \le
    \min\left\{ \frac{\e}{2}\rho(x,z)+\frac{\e}{2}\rho(z,y),\, M \right\}
    \le
    \ell(x,z)+\ell(z,y).
\]
    Moreover, for any $x,x',y,y' \in \cX$,
    \[
        \babs{ \ell(x,y)-\ell(x',y') }
        \le
        \babs{ \ell(x,y)-\ell(x',y) }
        +
        \babs{ \ell(x',y)-\ell(x',y') }
        \le
        \frac{\e}{2}\rho(x,x')
        +
        \frac{\e}{2}\rho(y,y') .
    \]
    Thus, $\ell$ is $\e$-stable.
    Then, define $f \colon \mathcal X \to \mathcal X$, for all $x\in \cX$, by
    \[
        f(x)
    \coloneqq
        \begin{cases}
        \varepsilon(1-x), & \text{if } x \in [0,\ 1],\\
        \varepsilon(x-1), & \text{if } x \in [1,\ 1+\nicefrac{2M}{\e^2}],\\
        \nicefrac{2M}{\e}, & \text{if } x \in [1+\nicefrac{2M}{\e^2},\ B],
        \end{cases}
    \]
    and $g \colon \mathcal X \to \mathcal X$, for all $x\in \cX$, by
    \[
        g(x)
    \coloneqq
        \begin{cases}
        0, & \text{if } x \in [0,1],\\
        \varepsilon(x-1), & \text{if } x \in [1,1+\nicefrac{2M}{\e^2}],\\
        \nicefrac{2M}{\e}, & \text{if } x \in [1+\nicefrac{2M}{\e^2}, B].
        \end{cases}
    \]
    By definition, both functions are $\e$-stable and 
    $
    \sup_{x \in \cX} \babs{ f(x) - g(x) } = \e
    $.
    Let the testing distribution be the Dirac measure at $1$, 
    $
        \nu \coloneqq \delta_1
    $.
    We now define the chain rule.
    Set, for all $x \in \cX$,
    \[
        D_K^{(1)}(x) \coloneqq 0.
    \]
    For the second step, define, for any $x,q_1,a_1 \in \cX$,
    \[
        D_K^{(2)}(x,q_1,a_1)
    \coloneqq
        \begin{cases}
        1, & a_1 = 0,\\
        1+\nicefrac{2M}{\e^2}, & a_1 \neq 0.
        \end{cases}
    \]
    For every $k \in \{3,\dots,K\}$, define, for any $x,q_1,a_1,\ldots,q_{k-1},a_{k-1} \in \cX$,
    \[
        D_K^{(k)}(x,q_1,a_1,\dots,q_{k-1},a_{k-1})
    \coloneqq
        q_{k-1}.
    \]
    Thus, after the second step, the chain rule simply repeats the previous question.
    
    We now compute the oracle's trajectory for the initial prompt $x=1$.
    Since
    \begin{align*}
    Q^{(1)}_{D_K,g}(1) & =0,
    \\
    A^{(1)}_{D_K,g}(1) &=g(0)=0,
    \end{align*}
    and the discontinuous branch in $D_K^{(2)}$ gives
    \begin{align*}
    Q^{(2)}_{D_K,g}(1) & =1,
    \\
    A^{(2)}_{D_K,g}(1) &=g(1)=0,
    \end{align*}
    then, for all later steps $k \in \{3, \dots, K\}$, since the chain rule repeats the previous question, we get
    \begin{align*}
    Q^{(k)}_{D_K,g}(1)& =1
    \\
    A^{(k)}_{D_K,g}(1)& =g(1)=0 .
    \end{align*}
    In particular, this yields
    \[
        A^{(K)}_{D_K,g}(1)=0=g(1),
    \]
    therefore $1$ is $(D_K,g)$-recoverable, and the support of $\nu$ is contained in the
    $(D_K,g)$-recoverable set.
    
    We now compute the learner's trajectory for the initial prompt $x=1$.
    Since
    \begin{align*}
    Q^{(1)}_{D_K,f}(1) &=0,
    \\
    A^{(1)}_{D_K,f}(1) &=f(0)=\varepsilon \neq 0,
    \end{align*}
    then, the second step of the chain rule follows the other branch:
    \begin{align*}
    Q^{(2)}_{D_K,f}(1) & =1+\nicefrac{2M}{\e^2},
    \\
    A^{(2)}_{D_K,f}(1) & =f(1+\nicefrac{2M}{\e^2})=\nicefrac{2M}{\e},
    \end{align*}
    and thus, for all later steps $k \in \{3, \dots, K\}$, since the chain rule repeats the previous question, we get
    \begin{align*}
    Q^{(k)}_{D_K,f}(1) & = 1+\nicefrac{2M}{\e^2}
    \\
    A^{(k)}_{D_K,f}(1) & =f(1+\nicefrac{2M}{\e^2})=\nicefrac{2M}{\e} .
    \end{align*}
    In particular, this yields
    \[
    A^{(K)}_{D_K,f}(1)=\nicefrac{2M}{\e}.
    \]
    Therefore,
    the reasoning risk is
    \[
        R^{D_K}_{\nu,\ell}(f,g)
        =
        \underset{X \sim \nu}{\mathbb{E}}
        \bbsb{
            \ell\Brb{
                A^{(K)}_{D_K,f}(X), \
                g(X)
            }
        }
    =
        \ell \brb{ \nicefrac{2M}{\e}, \ g(1) }
    =
        \min\left\{\frac{\e}{2}\labs{\nicefrac{2M}{\e}-0},M\right\}
    =
        M,
    \]
    the trajectory-mismatch risk is
    \[
        \underset{X \sim \nu}{\mathbb{E}}
        \bbsb{
            \ell\Brb{
                f \brb{ Q^{(K)}_{D_K,f}(X) },\ 
                f \brb{ Q^{(K)}_{D_K,g}(X) }
            }
        }
    =
        \ell \brb{ f(1+\nicefrac{2M}{\e^2}), \ f(1) }
    =
        \min\left\{\frac{\e}{2}\labs{\nicefrac{2M}{\e}-0},M\right\}
    =
        M,
    \]
    and the oracle-trajectory risk is
    \[
    \underset{X \sim \nu}{\mathbb{E}}
        \bbsb{
            \ell\Brb{
            f\brb{ Q^{(K)}_{D_K,g}(X) }, \
            g\brb{ Q^{(K)}_{D_K,g}(X) }
        }
    }
    =
        \ell \brb{ f(1),\ g(1) }
    =
        \min\left\{\frac{\e}{2}\labs{0-0},M\right\}
    =
        0 .
    \]
    This proves, in particular, that the right-hand side of the decomposition is $M+0=M$, which equals the
    reasoning risk, hence \eqref{eq:canonical-decomposition} holds with equality.
    \end{proof}

    \subsection{Proof for Theorem \ref{thm:stability-TMR}}
    \label[appendix]{appe:proofOfPositiveResult}
    
    Proving \Cref{thm:stability-TMR} requires the following combinatorial lemma, which gives a representation of the amplification factor as the solution of a maximization problem.

    \begin{lemma}[Max representation of the amplification factor]
    \label{lem:max-representation-amplification-factor}
    For any integer $K\ge 2$, any $\phi\ge 0$, and any $\delta\ge 0$, it holds that
    \[
        \alpha_K(\phi,\delta)
        =
        \max_{0\le m\le K-2}
        \delta^{K-2-m}
        \sum_{j=0}^{m}(\phi\delta)^j .
    \]
    (Here and below, every term with exponent $0$ is interpreted as $1$, including $0^0$.)
    \end{lemma}
    
    \begin{proof}
    For every $m\in\{0,\ldots,K-2\}$, define
    \[
        b_m
        \coloneqq
        \delta^{K-2-m}
        \sum_{j=0}^{m}(\phi\delta)^j .
    \]
    Let
    \[
        B
        \coloneqq
        \max_{0\le m\le K-2} b_m .
    \]
    We show that $B=\alpha_K(\phi,\delta)$.
    
    If $\delta=0$, then $\delta\le 1$ and $\phi\delta\neq 1$, so, by \Cref{def:amplification-factor},
    \[
        \alpha_K(\phi,0)
        =
        \frac{1-0^{K-1}}{1-0}
        =
        1.
    \]
    On the other hand, the maximum defining $B$ is attained at $m=K-2$, and equals $1$.
    Thus the claim holds when $\delta=0$.
    Hence, in the rest of the proof, assume $\delta>0$.
    
    For every $m\in\{0,\ldots,K-3\}$,
    \[
    \begin{aligned}
        b_{m+1}-b_m
        &=
        \delta^{K-3-m}
        \sum_{j=0}^{m+1}(\phi\delta)^j
        -
        \delta^{K-2-m}
        \sum_{j=0}^{m}(\phi\delta)^j                                      \\
        &=
        \delta^{K-3-m}
        \Brb{
            (1-\delta)\sum_{j=0}^{m}(\phi\delta)^j
            +
            (\phi\delta)^{m+1}
        } .
    \end{aligned}
    \]
    
    First suppose that $\delta\le 1$.
    Then $1-\delta\ge 0$, and therefore $b_{m+1}-b_m\ge 0$ for every $m$.
    Thus the maximum is attained at $m=K-2$.
    The same conclusion holds if $\phi\ge 1$.
    Indeed, when $\delta\le 1$ this was already proved, while if $\delta>1$ and $\phi\ge 1$, then $\phi\delta>1$, and
    \[
    \begin{aligned}
        (1-\delta)\sum_{j=0}^{m}(\phi\delta)^j
        +
        (\phi\delta)^{m+1}
        &=
        (1-\delta)\frac{(\phi\delta)^{m+1}-1}{\phi\delta-1}
        +
        (\phi\delta)^{m+1}                                                \\
        &=
        \frac{
            \delta(\phi-1)(\phi\delta)^{m+1}
            +
            \delta-1
        }{\phi\delta-1}
        \ge 0 .
    \end{aligned}
    \]
    Therefore, if $\phi\ge 1$ or $\delta\le 1$, then
    \[
        B
        =
        b_{K-2}
        =
        \sum_{j=0}^{K-2}(\phi\delta)^j .
    \]
    If $\phi\delta\neq 1$, this gives
    \[
        B
        =
        \frac{1-(\phi\delta)^{K-1}}{1-\phi\delta}
        =
        \alpha_K(\phi,\delta).
    \]
    If $\phi\delta=1$, this gives
    \[
        B
        =
        K-1
        =
        \alpha_K(\phi,\delta).
    \]
    
    It remains to consider the case $\phi<1$ and $\delta>1$.
    First suppose that $\phi=0$.
    Then, for every $m\in\{0,\ldots,K-2\}$,
    \[
        b_m
        =
        \delta^{K-2-m},
    \]
    because $\sum_{j=0}^{m}(\phi\delta)^j=1$.
    Since $\delta>1$, the maximum is attained at $m=0$, and therefore
    \[
        B
        =
        \delta^{K-2}.
    \]
    Since $m_K(0,\delta)=0$, this is exactly
    \[
        B
        =
        \delta^{K-2-m_K(0,\delta)}
        \frac{1-(\phi\delta)^{m_K(0,\delta)+1}}{1-\phi\delta}
        =
        \alpha_K(0,\delta).
    \]
    
    Now suppose that $\phi\in(0,1)$ and $\delta>1$.
    Set
    \[
        r\coloneqq \phi\delta .
    \]
    Assume first that $r\neq 1$.
    Using the geometric-sum formula, for every $m\in\{0,\ldots,K-3\}$,
    \[
    \begin{aligned}
        b_{m+1}-b_m
        &=
        \delta^{K-3-m}
        \Brb{
            (1-\delta)\sum_{j=0}^{m}r^j
            +
            r^{m+1}
        } .
    \end{aligned}
    \]
    If $r<1$, then
    \[
    \begin{aligned}
        (1-\delta)\sum_{j=0}^{m}r^j+r^{m+1}
        &=
        (1-\delta)\frac{1-r^{m+1}}{1-r}+r^{m+1}                            \\
        &=
        \frac{
            1-\delta
            +
            \delta(1-\phi)r^{m+1}
        }{1-r}.
    \end{aligned}
    \]
    Since $1-r>0$, we have $b_{m+1}\ge b_m$ if and only if
    \[
        r^{m+1}
        \ge
        \frac{\delta-1}{\delta(1-\phi)}.
    \]
    If $r>1$, then
    \[
    \begin{aligned}
        (1-\delta)\sum_{j=0}^{m}r^j+r^{m+1}
        &=
        (1-\delta)\frac{r^{m+1}-1}{r-1}+r^{m+1}                              \\
        &=
        \frac{
            \delta-1
            -
            \delta(1-\phi)r^{m+1}
        }{r-1}.
    \end{aligned}
    \]
    Since $r-1>0$, we again get $b_{m+1}\ge b_m$ if and only if
    \[
        r^{m+1}
        \le
        \frac{\delta-1}{\delta(1-\phi)}.
    \]
    Moreover,
    \[
        r<1
        \quad\Longleftrightarrow\quad
        \frac{\delta-1}{\delta(1-\phi)}<1,
    \]
    and
    \[
        r>1
        \quad\Longleftrightarrow\quad
        \frac{\delta-1}{\delta(1-\phi)}>1.
    \]
    Thus, in both cases, the sequence $(b_m)$ is increasing up to
    \[
        m_K(\phi,\delta)
        =
        \min\bbrb{
            K-2,
            \bbfl{
            \log_{\phi\delta}\Brb{\frac{\delta-1}{\delta(1-\phi)}}
            }
        }
    \]
    and decreasing afterwards.
    Hence the maximum is attained at $m_K(\phi,\delta)$.
    Since $r=\phi\delta\neq 1$, the geometric-sum formula gives
    \[
    \begin{aligned}
        B
        &=
        b_{m_K(\phi,\delta)}                                                \\
        &=
        \delta^{K-2-m_K(\phi,\delta)}
        \frac{1-(\phi\delta)^{m_K(\phi,\delta)+1}}{1-\phi\delta}
        =
        \alpha_K(\phi,\delta).
    \end{aligned}
    \]
    
    Finally, suppose that $\phi<1$, $\delta>1$, and $\phi\delta=1$.
    Then necessarily $\phi>0$, and
    \[
        b_m
        =
        \delta^{K-2-m}(m+1).
    \]
    Moreover,
    \[
    \begin{aligned}
        b_{m+1}-b_m
        &=
        \delta^{K-3-m}(m+2)-\delta^{K-2-m}(m+1)                                 \\
        &=
        \delta^{K-3-m}
        \Brb{
            1-(\delta-1)(m+1)
        } .
    \end{aligned}
    \]
    Therefore $b_{m+1}\ge b_m$ if and only if
    \[
        m+1\le \frac{1}{\delta-1}.
    \]
    Hence the maximum is attained at
    \[
        n_K(\phi,\delta)
        =
        \min\bbrb{
            K-2,\
            \bbfl{ \frac{1}{\delta-1} }
        },
    \]
    and
    \[
    \begin{aligned}
        B
        &=
        b_{n_K(\phi,\delta)}                                                \\
        &=
        \delta^{K-2-n_K(\phi,\delta)}
        \brb{n_K(\phi,\delta)+1}
        =
        \alpha_K(\phi,\delta).
    \end{aligned}
    \]
    This proves the claim.
    \end{proof}

    We can now prove \Cref{thm:stability-TMR}.

    \begin{proof}[Proof of Theorem \ref{thm:stability-TMR}]
    Write
    \[
        d:=d(f,g)=\sup_{z\in\mathcal X}\rho(f(z),g(z)).
    \]
    If $d=0$, then $f=g$ pointwise.
    Hence the trajectories generated by $f$ and $g$ coincide at every step.
    Since $\ell$ is a quasimetric, it is zero on the diagonal, and therefore the TMR term is equal to $0$.
    Thus the desired inequality is immediate.
    Assume from now on that $d>0$.
    
    Fix $x\in\mathcal X$.
    For $k=1,\ldots,K$, define
    \[
        \Delta_k(x)
        :=
        \rho\!\left(
            Q^{(k)}_{D_K,f}(x),
            Q^{(k)}_{D_K,g}(x)
        \right),
    \]
    and
    \[
        \Gamma_k(x)
        :=
        \rho\!\left(
            A^{(k)}_{D_K,f}(x),
            A^{(k)}_{D_K,g}(x)
        \right).
    \]
    Since
    \[
        Q^{(1)}_{D_K,f}=Q^{(1)}_{D_K,g}=D_K^{(1)},
    \]
    we have
    \[
        \Delta_1(x)=0.
    \]
    
    For every $i=1,\ldots,K$, by the triangle inequality for $\rho$, the
    $\phi$-stability of $f$, and the definition of $d$, we have
    \[
    \begin{aligned}
        \Gamma_i(x)
        &=
        \rho\!\left(
            f(Q^{(i)}_{D_K,f}(x)),
            g(Q^{(i)}_{D_K,g}(x))
        \right)                                                     \\
        &\le
        \rho\!\left(
            f(Q^{(i)}_{D_K,f}(x)),
            f(Q^{(i)}_{D_K,g}(x))
        \right)
        +
        \rho\!\left(
            f(Q^{(i)}_{D_K,g}(x)),
            g(Q^{(i)}_{D_K,g}(x))
        \right)                                                     \\
        &\le
        \phi\,\Delta_i(x)+d .
    \end{aligned}
    \]
    Now fix $k\in\{2,\ldots,K\}$.
    Since $D_K^{(k)}$ is $\delta$-stable, there exist stability constants
    \[
        \delta_0^{(k)},\delta_1^{(k)},\ldots,\delta_{2k-2}^{(k)}\ge 0
    \]
    such that
    \[
        \sum_{r=0}^{2k-2}\delta_r^{(k)}\le \delta.
    \]
    Applying the stability condition to the two inputs
    \[
        \Brb{
            x,
            Q^{(1)}_{D_K,f}(x),
            A^{(1)}_{D_K,f}(x),
            \ldots,
            Q^{(k-1)}_{D_K,f}(x),
            A^{(k-1)}_{D_K,f}(x)
        }
    \]
    and
    \[
        \Brb{
            x,
            Q^{(1)}_{D_K,g}(x),
            A^{(1)}_{D_K,g}(x),
            \ldots,
            Q^{(k-1)}_{D_K,g}(x),
            A^{(k-1)}_{D_K,g}(x)
        },
    \]
    we get
    \[
    \begin{aligned}
        &\Delta_k(x)
    \\
        &=
        \rho\Brb{
            Q^{(k)}_{D_K,f}(x),
            Q^{(k)}_{D_K,g}(x)
        }                                                          \\
        &=
        \rho\Brb{
            D_K^{(k)}\Brb{
                x,
                Q^{(1)}_{D_K,f}(x),
                A^{(1)}_{D_K,f}(x),
                \ldots,
                Q^{(k-1)}_{D_K,f}(x),
                A^{(k-1)}_{D_K,f}(x)
            },
            D_K^{(k)}\Brb{
                x,
                Q^{(1)}_{D_K,g}(x),
                A^{(1)}_{D_K,g}(x),
                \ldots,
                Q^{(k-1)}_{D_K,g}(x),
                A^{(k-1)}_{D_K,g}(x)
            }
        }                                                          \\
        &\le
        \sum_{i=1}^{k-1}
            \delta_{2i-1}^{(k)}
            \rho\Brb{
                Q^{(i)}_{D_K,f}(x),
                Q^{(i)}_{D_K,g}(x)
            }
        +
        \sum_{i=1}^{k-1}
            \delta_{2i}^{(k)}
            \rho\Brb{
                A^{(i)}_{D_K,f}(x),
                A^{(i)}_{D_K,g}(x)
            }                                                       \\
        &=
        \sum_{i=1}^{k-1}
            \delta_{2i-1}^{(k)}\Delta_i(x)
        +
        \sum_{i=1}^{k-1}
            \delta_{2i}^{(k)}\Gamma_i(x)                            \\
        &\le
        \Brb{
            \sum_{i=1}^{k-1}
            \delta_{2i-1}^{(k)}
            +
            \sum_{i=1}^{k-1}
            \delta_{2i}^{(k)}
        }
        \max_{1\le i\le k-1}
        \max\{\Delta_i(x),\Gamma_i(x)\}                              \\
        &\le
        \delta\max_{1\le i\le k-1}
        \max\{\Delta_i(x),\Gamma_i(x)\}.
    \end{aligned}
    \]
    Together with $\Gamma_i(x)\le \phi\Delta_i(x)+d$, this gives
    \[
        \Delta_k(x)
        \le
        \delta\max_{1\le i\le k-1}
        \max\{\Delta_i(x),\phi\Delta_i(x)+d\}.
    \]
    Define
    \[
        e_k(x):=\frac{\Delta_k(x)}{d}.
    \]
    Then
    \[
        e_1(x)=0,
    \]
    and, for every $k=2,\ldots,K$,
    \[
        e_k(x)
        \le
        \delta\max_{1\le i\le k-1}
        \max\{e_i(x),\phi e_i(x)+1\}.
    \]
    
    Introduce the two increasing affine maps
    \[
        T_{\mathsf Q}(z):=\delta z,
        \qquad
        T_{\mathsf A}(z):=\delta(\phi z+1).
    \]
    For $L\ge 0$, let $\mathcal W_L$ be the set of all words of length at most
    $L$ over the alphabet $\{\mathsf Q,\mathsf A\}$.
    For a word $w=(w_1,\ldots,w_m)$, write
    \[
        T_w:=T_{w_m}\circ\cdots\circ T_{w_1},
    \]
    and set $T_{\varnothing}(0)=0$.
    Define
    \[
        C_L:=\max_{w\in\mathcal W_L}T_w(0).
    \]
    We claim that
    \[
        e_k(x)\le C_{k-1}
        \qquad
        \text{for every } k=1,\ldots,K.
    \]
    The claim is true for $k=1$, since $e_1(x)=0=C_0$.
    Suppose the claim holds for all indices smaller than $k$.
    Then, for every $i\in\{1,\ldots,k-1\}$, the induction hypothesis gives
    \[
        e_i(x)\le C_{i-1}.
    \]
    Since $i-1\le k-2$ and $C_L$ is nondecreasing in $L$, we also have
    \[
        C_{i-1}\le C_{k-2}.
    \]
    Therefore,
    \[
        e_i(x)\le C_{k-2}
        \qquad
        \text{for every } i\in\{1,\ldots,k-1\}.
    \]
    Since $T_{\mathsf Q}$ and $T_{\mathsf A}$ are increasing, this implies
    \[
        T_{\mathsf Q}(e_i(x))\le T_{\mathsf Q}(C_{k-2}),
        \qquad
        T_{\mathsf A}(e_i(x))\le T_{\mathsf A}(C_{k-2})
    \]
    for every $i\in\{1,\ldots,k-1\}$.
    Recalling that
    \[
        T_{\mathsf Q}(z)=\delta z,
        \qquad
        T_{\mathsf A}(z)=\delta(\phi z+1),
    \]
    we get
    \[
    \begin{aligned}
        e_k(x)
        &\le
        \delta\max_{1\le i\le k-1}
        \max\{e_i(x),\phi e_i(x)+1\}                                  \\
        &=
        \max_{1\le i\le k-1}
        \max\{T_{\mathsf Q}(e_i(x)),T_{\mathsf A}(e_i(x))\}             \\
        &\le
        \max\{T_{\mathsf Q}(C_{k-2}),T_{\mathsf A}(C_{k-2})\}.
    \end{aligned}
    \]
    It remains to prove that
    \[
        \max\{T_{\mathsf Q}(C_{k-2}),T_{\mathsf A}(C_{k-2})\}\le C_{k-1}.
    \]
    Indeed, since $\mathcal W_{k-2}$ is finite, there exists a word
    $w^\star=(w_1,\ldots,w_m)\in\mathcal W_{k-2}$ such that
    \[
        C_{k-2}=T_{w^\star}(0).
    \]
    Since $w^\star\in\mathcal W_{k-2}$, its length satisfies $m\le k-2$.
    Therefore, the two words
    \[
        w_{\mathsf Q}^\star \coloneqq (w_1,\ldots,w_m,\mathsf Q),
        \qquad
        w_{\mathsf A}^\star \coloneqq (w_1,\ldots,w_m,\mathsf A)
    \]
    have length at most $k-1$, and hence both belong to $\mathcal W_{k-1}$.
    By the definition of $T_w$, we have
    \[
        T_{w_{\mathsf Q}^\star}
        =
        T_{\mathsf Q}\circ T_{w^\star},
        \qquad
        T_{w_{\mathsf A}^\star}
        =
        T_{\mathsf A}\circ T_{w^\star}.
    \]
    Evaluating at $0$ gives
    \[
        T_{w_{\mathsf Q}^\star}(0)
        =
        T_{\mathsf Q}\brb{T_{w^\star}(0)}
        =
        T_{\mathsf Q}(C_{k-2}),
    \]
    and similarly
    \[
        T_{w_{\mathsf A}^\star}(0)
        =
        T_{\mathsf A}\brb{T_{w^\star}(0)}
        =
        T_{\mathsf A}(C_{k-2}).
    \]
    Since $C_{k-1}$ is the maximum of $T_w(0)$ over all $w\in\mathcal W_{k-1}$, and since
    $w_{\mathsf Q}^\star,w_{\mathsf A}^\star\in\mathcal W_{k-1}$, we conclude that
    \[
        T_{\mathsf Q}(C_{k-2})\le C_{k-1},
        \qquad
        T_{\mathsf A}(C_{k-2})\le C_{k-1}.
    \]
    Therefore,
    \[
        \max\{T_{\mathsf Q}(C_{k-2}),T_{\mathsf A}(C_{k-2})\}\le C_{k-1}.
    \]
    Hence
    \[
        e_k(x)\le C_{k-1}.
    \]
    Thus, in particular,
    \[
        \Delta_K(x)\le d\,C_{K-1}.
    \]
    
    It remains to compute $C_{K-1}$.
    Let $L:=K-1$.
    First suppose $\delta\ge 1$.
    For every $z\ge 0$, we have
    \[
    \begin{aligned}
        (T_{\mathsf Q}\circ T_{\mathsf A})(z)-(T_{\mathsf A}\circ T_{\mathsf Q})(z)
        &=
        T_{\mathsf Q}\brb{\delta(\phi z+1)}-T_{\mathsf A}(\delta z)                                  \\
        &=
        \delta^2(\phi z+1)-\delta(\phi \delta z+1)                                     \\
        &=
        \delta^2\phi z+\delta^2-\delta^2\phi z-\delta                                      \\
        &=
        \delta(\delta-1)
        \ge 0.
    \end{aligned}
    \]
    Thus, applying $T_{\mathsf A}$ first and then $T_{\mathsf Q}$ gives a value at least as large as applying
    $T_{\mathsf Q}$ first and then $T_{\mathsf A}$.
    
    We now use this as an adjacent-swap argument.
    Take any word $w$ with a fixed number of $\mathsf A$'s and $\mathsf Q$'s.
    If $w$ contains an adjacent pair $(\mathsf Q,\mathsf A)$, write
    \[
        w=(u,\mathsf Q,\mathsf A,v),
    \]
    where $u$ is the part of the word before this adjacent pair and $v$ is the part after it.
    Let
    \[
        \widetilde w=(u,\mathsf A,\mathsf Q,v)
    \]
    be the word obtained by swapping this adjacent pair.
    By the definition of $T_w$, we have
    \[
        T_w
        =
        T_v\circ T_{\mathsf A}\circ T_{\mathsf Q}\circ T_u,
        \qquad
        T_{\widetilde w}
        =
        T_v\circ T_{\mathsf Q}\circ T_{\mathsf A}\circ T_u.
    \]
    Set
    \[
        z\coloneqq T_u(0).
    \]
    Since $T_{\mathsf Q}$ and $T_{\mathsf A}$ map nonnegative numbers to nonnegative numbers, we have $z\ge 0$.
    Therefore,
    \[
        (T_{\mathsf Q}\circ T_{\mathsf A})(z)\ge (T_{\mathsf A}\circ T_{\mathsf Q})(z).
    \]
    Moreover, $T_v$ is increasing, being a composition of increasing maps.
    Hence
    \[
        T_{\widetilde w}(0)
        =
        T_v\Brb{(T_{\mathsf Q}\circ T_{\mathsf A})(z)}
        \ge
        T_v\Brb{(T_{\mathsf A}\circ T_{\mathsf Q})(z)}
        =
        T_w(0).
    \]
    Thus, whenever a word contains an adjacent pair $(\mathsf Q,\mathsf A)$, swapping it to $(\mathsf A,\mathsf Q)$ cannot
    decrease its value at $0$.
    
    Repeating this adjacent swap finitely many times moves all occurrences of $\mathsf A$ before all
    occurrences of $\mathsf Q$.
    Consequently, among words with a fixed number of $\mathsf A$'s and $\mathsf Q$'s, the value at $0$ is
    maximized by a word of the form
    \[
        (\mathsf A,\ldots,\mathsf A,\mathsf Q,\ldots,\mathsf Q),
    \]
    that is, by applying all $T_{\mathsf A}$'s first and all $T_{\mathsf Q}$'s afterwards.
    Moreover, since $\delta\ge 1$, any shorter word can be extended by appending $\mathsf Q$'s without
    decreasing its value.
    Therefore
    \[
        C_L
        =
        \max_{1\le n\le L}
        T_{\mathsf Q}^{L-n}T_{\mathsf A}^n(0).
    \]
    We now compute $T_{\mathsf A}^n(0)$.
    Since
    \[
        T_{\mathsf A}(z)=\delta(\phi z+1)=\delta\phi z+\delta,
    \]
    we claim that, for every $n\ge 1$,
    \[
        T_{\mathsf A}^n(0)
        =
        \delta\sum_{j=0}^{n-1}(\delta\phi)^j .
    \]
    For $n=1$, this gives
    \[
        T_{\mathsf A}(0)=\delta
        =
        \delta\sum_{j=0}^{0}(\delta\phi)^j .
    \]
    Assume now that the identity holds for some $n\ge 1$.
    Then
    \[
    \begin{aligned}
        T_{\mathsf A}^{n+1}(0)
        &=
        T_{\mathsf A}\brb{T_{\mathsf A}^n(0)}                                             \\
        &=
        \delta\phi T_{\mathsf A}^n(0)+\delta                                               \\
        &=
        \delta\phi
        \brb{
            \delta\sum_{j=0}^{n-1}(\delta\phi)^j
        }
        +\delta                                                            \\
        &=
        \delta\sum_{j=1}^{n}(\delta\phi)^j+\delta                                    \\
        &=
        \delta\sum_{j=0}^{n}(\delta\phi)^j .
    \end{aligned}
    \]
    Thus, by induction,
    \[
        T_{\mathsf A}^n(0)
        =
        \delta\sum_{j=0}^{n-1}(\delta\phi)^j .
    \]
    Consequently, we get
    \[
        C_L
        =
        \max_{1\le n\le L}
        \delta^{L-n+1}
        \sum_{j=0}^{n-1}(\delta\phi)^j.
    \]
    Equivalently, with $m=n-1$,
    \[
        C_{K-1}
        =
        \max_{0\le m\le K-2}
        \delta^{K-1-m}
        \sum_{j=0}^{m}(\delta\phi)^j.
    \]
    
    Now suppose $0\le \delta<1$.
    Then
    \[
        (T_{\mathsf Q}\circ T_{\mathsf A})(z)-(T_{\mathsf A}\circ T_{\mathsf Q})(z)
        =
        \delta(\delta-1)
        \le 0.
    \]
    Thus, another adjacent-swap argument gives that moving all $\mathsf Q$'s before all $\mathsf A$'s can only increase the value at $0$.
    Since $T_{\mathsf Q}(0)=0$, any initial block of $\mathsf Q$'s has no effect.
    Therefore, for a word containing $n$ occurrences of $\mathsf A$,
    \[
        T_w(0)\le T_{\mathsf A}^n(0)
        =
        \delta\sum_{j=0}^{n-1}(\delta\phi)^j.
    \]
    This quantity is nondecreasing in $n$, so the maximum over words of length
    at most $L$ is attained by $T_{\mathsf A}^L$.
    Hence
    \[
        C_L
        =
        T_{\mathsf A}^L(0)
        =
        \delta\sum_{j=0}^{L-1}(\delta\phi)^j.
    \]
    This is precisely the value of
    \[
        \max_{0\le m\le K-2}
        \delta^{K-1-m}
        \sum_{j=0}^{m}(\delta\phi)^j
    \]
    at $m=K-2$.
    Therefore, in all cases,
    \[
        C_{K-1}
        =
        \max_{0\le m\le K-2}
        \delta^{K-1-m}
        \sum_{j=0}^{m}(\delta\phi)^j.
    \]
    By \Cref{lem:max-representation-amplification-factor},
    \[
        C_{K-1}
        =
        \delta\,\alpha_K(\phi,\delta).
    \]
    Consequently, for every $x\in\mathcal X$,
    \[
        \Delta_K(x)
        \le
        \delta\,d(f,g)\,\alpha_K(\phi,\delta).
    \]
    
    We now bound the trajectory-mismatch risk.
    Because $\ell$ is $\lambda$-stable and is zero on the diagonal, there exist
    $\lambda_1,\lambda_2\ge 0$ such that $\lambda_1+\lambda_2\le\lambda$ and
    \[
        \labs{\ell(u,v)-\ell(u',v')}
        \le
        \lambda_1\rho(u,u')+\lambda_2\rho(v,v')
    \]
    for all $u,u',v,v'\in\cX$.
    Thus, for all $u,v\in\cX$,
    \[
        \ell(u,v)
        =
        \labs{\ell(u,v)-\ell(v,v)}
        \le
        \lambda_1\rho(u,v),
    \]
    and also
    \[
        \ell(u,v)
        =
        \labs{\ell(u,v)-\ell(u,u)}
        \le
        \lambda_2\rho(u,v).
    \]
    Hence
    \[
        \ell(u,v)
        \le
        \min\{\lambda_1,\lambda_2\}\rho(u,v)
        \le
        \frac{\lambda}{2}\rho(u,v).
    \]
    Therefore, for every $x\in\cX$,
    \[
    \begin{aligned}
        \ell\Brb{
            f\brb{Q_{D_K,f}^{(K)}(x)},
            f\brb{Q_{D_K,g}^{(K)}(x)}
        }
        &\le
        \frac{\lambda}{2}
        \rho\Brb{
            f\brb{Q_{D_K,f}^{(K)}(x)},
            f\brb{Q_{D_K,g}^{(K)}(x)}
        }                                                       \\
        &\le
        \frac{\lambda\phi}{2}
        \rho\Brb{
            Q_{D_K,f}^{(K)}(x),
            Q_{D_K,g}^{(K)}(x)
        }                                                       \\
        &=
        \frac{\lambda\phi}{2}\Delta_K(x)                         \\
        &\le
        \frac{\lambda\phi\delta}{2}
        d(f,g)\alpha_K(\phi,\delta).
    \end{aligned}
    \]
    Taking expectation with respect to $X\sim\nu$ yields
    \[
        \underset{X \sim \nu}{\E}
            \bbsb{
                \ell\Brb{
                    f\brb{ Q_{D_K,f}^{(K)}(X) }, \,
                    f\brb{ Q_{D_K,g}^{(K)}(X) }
                }
            }
        \le
        \frac{\lambda\phi\delta}{2}\,
        \alpha_K(\phi,\delta)\,
        d(f,g).
    \]
    This proves the theorem.
    \end{proof}

\subsection{Proof for Theorem \ref{thm:stability-reasoning-risk-tight}}
\label[appendix]{appe:proofOfOptimalityOfPositiveResult}

\begin{proof}[Proof of Theorem \ref{thm:stability-reasoning-risk-tight}]
Let
\[
    m_\star
    \in
    \argmax_{0\le m\le K-2}
    \delta^{K-2-m}
    \sum_{j=0}^{m}(\phi\delta)^j ,
\]
where powers with exponent $0$ (including $0^0$) are interpreted as equal to $1$.
Consider the word $w=(w_1,\ldots,w_{K-1})\in\{\mathsf A,\mathsf Q\}^{K-1}$ defined by
\[
    w_r
    \coloneqq
    \begin{cases}
    \mathsf A, & \text{if } r\le m_\star+1,\\
    \mathsf Q, & \text{if } r>m_\star+1.
    \end{cases}
\]
Let
\[
    T_{\mathsf Q}(z)\coloneqq \delta z,
    \qquad
    T_{\mathsf A}(z)\coloneqq \delta(\phi z+1),
\]
and define the sequence $(E_r)_{r=0}^{K-1}$ by
\[
    E_0\coloneqq 0,
    \qquad
    E_r\coloneqq T_{w_r}(E_{r-1})
    \quad
    \text{for every } r\in[K-1].
\]
By construction,
\[
\begin{aligned}
    E_{K-1}
    &=
    \delta^{K-1-m_\star}
    \sum_{j=0}^{m_\star}(\phi\delta)^j                                      \\
    &=
    \delta
    \max_{0\le m\le K-2}
    \delta^{K-2-m}
    \sum_{j=0}^{m}(\phi\delta)^j                                             \\
    &=
    \delta\alpha_K(\phi,\delta),
\end{aligned}
\]
where the last equality follows from \Cref{lem:max-representation-amplification-factor}.

Let $\cX\coloneqq\R$, equipped with the usual metric
\[
    \rho(x,y)\coloneqq \labs{x-y}.
\]
Choose distinct real numbers $s_1,\ldots,s_K$.
Define
\[
    f(z)\coloneqq \phi z
    \qquad
    \text{for all } z\in\R.
\]
Define $g\colon\R\to\R$ by
\[
    g(s_i)\coloneqq f(s_i)-1
    \quad
    \text{for every } i\in\{1,\ldots,K-1\},
\]
\[
    g(s_K)\coloneqq f(s_K),
\]
and
\[
    g(z)\coloneqq f(z)
    \quad
    \text{for every } z\notin\{s_1,\ldots,s_K\}.
\]
Then
\[
    d(f,g)
    =
    \sup_{z\in\R}\rho\brb{f(z),g(z)}
    =
    1.
\]
Moreover, $f$ is $\phi$-stable.

Define the loss function $\ell\colon\R\times\R\to[0,\infty)$ by
\[
    \ell(u,v)\coloneqq \frac{\lambda}{2}\labs{u-v}.
\]
Since $\ell$ is a positive multiple of the metric $\rho$, it is a quasimetric.
Moreover, for any $u,u',v,v'\in\R$,
\[
\begin{aligned}
    \labs{\ell(u,v)-\ell(u',v')}
    &\le
    \labs{\ell(u,v)-\ell(u',v)}
    +
    \labs{\ell(u',v)-\ell(u',v')}       \\
    &\le
    \frac{\lambda}{2}\rho(u,u')
    +
    \frac{\lambda}{2}\rho(v,v').
\end{aligned}
\]
Thus $\ell$ is $\lambda$-stable, with coordinate-wise Lipschitz constants
$\lambda/2$ and $\lambda/2$.

Let the testing distribution be the Dirac measure at $s_K$,
\[
    \nu\coloneqq \delta_{s_K}.
\]
We now define the chain rule.
Set
\[
    D_K^{(1)}(x)\coloneqq s_1
    \qquad
    \text{for all } x\in\R.
\]
For every $k\in\{2,\ldots,K\}$, write $r=k-1$.
If $w_r=\mathsf Q$, define
\[
    D_K^{(k)}(x,q_1,a_1,\ldots,q_{k-1},a_{k-1})
    \coloneqq
    s_k-\delta s_{k-1}+\delta q_{k-1}.
\]
If $w_r=\mathsf A$, define
\[
    D_K^{(k)}(x,q_1,a_1,\ldots,q_{k-1},a_{k-1})
    \coloneqq
    s_k-\delta g(s_{k-1})+\delta a_{k-1}.
\]
Each step of $D_K$ is $\delta$-stable.
Indeed, $D_K^{(1)}$ is constant.
For $k\ge 2$, the map $D_K^{(k)}$ depends on exactly one previous coordinate, either $q_{k-1}$ or $a_{k-1}$, with Lipschitz constant $\delta$, and does not depend on any other coordinate.

We now compute the oracle trajectory.
Since $D_K^{(1)}(s_K)=s_1$, we have
\[
    Q^{(1)}_{D_K,g}(s_K)=s_1.
\]
We claim that, for every $k\in[K]$,
\[
    Q^{(k)}_{D_K,g}(s_K)=s_k.
\]
This is true for $k=1$.
Suppose it holds up to step $k-1$.
If $w_{k-1}=\mathsf Q$, then
\[
    Q^{(k)}_{D_K,g}(s_K)
    =
    s_k-\delta s_{k-1}+\delta Q^{(k-1)}_{D_K,g}(s_K)
    =
    s_k.
\]
If $w_{k-1}=\mathsf A$, then
\[
    Q^{(k)}_{D_K,g}(s_K)
    =
    s_k-\delta g(s_{k-1})+\delta A^{(k-1)}_{D_K,g}(s_K)
    =
    s_k,
\]
because
\[
    A^{(k-1)}_{D_K,g}(s_K)
    =
    g\brb{Q^{(k-1)}_{D_K,g}(s_K)}
    =
    g(s_{k-1}).
\]
Thus the claim holds for all $k$.
In particular,
\[
    Q^{(K)}_{D_K,g}(s_K)=s_K.
\]
Since
\[
    A^{(K)}_{D_K,g}(s_K)
    =
    g\brb{Q^{(K)}_{D_K,g}(s_K)}
    =
    g(s_K),
\]
we have
\[
    A^{(K)}_{D_K,g}(s_K)=g(s_K).
\]
Therefore $s_K$ is $(D_K,g)$-recoverable, and $\nu$ is supported on the $(D_K,g)$-recoverable set.

We now compute the learner trajectory.
We claim that, for every $k\in[K]$,
\[
    Q^{(k)}_{D_K,f}(s_K)
    =
    s_k+E_{k-1}.
\]
For $k=1$, this follows from $E_0=0$ and
\[
    Q^{(1)}_{D_K,f}(s_K)=s_1.
\]
Suppose the claim holds up to step $k-1$.
If $w_{k-1}=\mathsf Q$, then
\[
\begin{aligned}
    Q^{(k)}_{D_K,f}(s_K)
    &=
    s_k-\delta s_{k-1}+\delta Q^{(k-1)}_{D_K,f}(s_K)       \\
    &=
    s_k-\delta s_{k-1}+\delta\brb{s_{k-1}+E_{k-2}}        \\
    &=
    s_k+\delta E_{k-2}
    =
    s_k+E_{k-1}.
\end{aligned}
\]
If $w_{k-1}=\mathsf A$, then, since $k-1\le K-1$, we have
\[
    g(s_{k-1})=f(s_{k-1})-1.
\]
Therefore
\[
\begin{aligned}
    Q^{(k)}_{D_K,f}(s_K)
    &=
    s_k-\delta g(s_{k-1})+\delta A^{(k-1)}_{D_K,f}(s_K)     \\
    &=
    s_k-\delta g(s_{k-1})
    +
    \delta f\brb{Q^{(k-1)}_{D_K,f}(s_K)}             \\
    &=
    s_k-\delta\brb{f(s_{k-1})-1}
    +
    \delta f\brb{s_{k-1}+E_{k-2}}                    \\
    &=
    s_k-\delta\brb{\phi s_{k-1}-1}
    +
    \delta\phi\brb{s_{k-1}+E_{k-2}}                  \\
    &=
    s_k+\delta\brb{\phi E_{k-2}+1}
    =
    s_k+E_{k-1}.
\end{aligned}
\]
Thus the learner-trajectory claim holds for all $k$.
In particular,
\[
    Q^{(K)}_{D_K,f}(s_K)
    =
    s_K+E_{K-1}
    =
    s_K+\delta\alpha_K(\phi,\delta).
\]

We can now evaluate the two terms.
First, the oracle-trajectory risk is
\[
\begin{aligned}
    R_{Q^{(K)}_{D_K,g}\#\nu,\ell}(f,g)
    &=
    \ell\Brb{
        f\brb{Q^{(K)}_{D_K,g}(s_K)},
        g\brb{Q^{(K)}_{D_K,g}(s_K)}
    }                                                   \\
    &=
    \ell\brb{f(s_K),g(s_K)}
    =
    \ell\brb{f(s_K),f(s_K)}
    =
    0.
\end{aligned}
\]
Second, the reasoning risk is
\[
\begin{aligned}
    R^{D_K}_{\nu,\ell}(f,g)
    &=
    \ell\Brb{
        A^{(K)}_{D_K,f}(s_K),
        g(s_K)
    }                                                   \\
    &=
    \ell\Brb{
        f\brb{Q^{(K)}_{D_K,f}(s_K)},
        g(s_K)
    }                                                   \\
    &=
    \ell\brb{
        f\brb{s_K+\delta\alpha_K(\phi,\delta)},
        f(s_K)
    }                                                   \\
    &=
    \frac{\lambda}{2}
    \rho\brb{
        f\brb{s_K+\delta\alpha_K(\phi,\delta)},
        f(s_K)
    }                                                   \\
    &=
    \frac{\lambda\phi\delta}{2}
    \alpha_K(\phi,\delta).
\end{aligned}
\]
Since $d(f,g)=1$ and the oracle-trajectory risk is zero, this is exactly
\[
    \frac{\lambda\phi\delta}{2}
    d(f,g)
    \alpha_K(\phi,\delta)
    +
    R_{Q^{(K)}_{D_K,g}\#\nu,\ell}(f,g).
\]
Thus the upper bound of \Cref{cor:stability-reasoning-risk} is attained.
\end{proof}

\section{Additional material for Section \ref{sec:formalizing-cot}}

\subsection{Details for Example \ref{ex:multiply}}
\label[appendix]{app:details-example-multiply}

We explicitly specify the four functions in 
$D_4=\brb{D_4^{(1)},D_4^{(2)},D_4^{(3)},D_4^{(4)}}$
for the illustrative case of multiplying a one-digit integer by a two-digit integer.
Throughout this example, we slightly abuse notation by identifying a nonnegative integer with its decimal-string representation, which is an element of $\cX$.
Let the input be the two-factor multiplication expression
\[
    x = d_1 \cdot n,
    \qquad
    n = 10d_2+d_3,
\]
where $d_1,d_3\in\set{0,\dots,9}$ and $d_2\in\set{1,\dots,9}$.
We define the four steps of the chain rule on such inputs as follows:
\begin{enumerate}
    \item The first step extracts the tens digit of the two-digit factor and produces the elementary multiplication
    \[
        D_4^{(1)}(x)
        \coloneqq
        d_1\cdot d_2.
    \]

    \item The second step multiplies the first partial product by $10$.
    For generic coordinates $(x,q_1,a_1)\in\cX^3$, define
    \[
        D_4^{(2)}(x,q_1,a_1)
        \coloneqq
        10\cdot a_1,
    \]
    whenever $x=d_1\cdot(10d_2+d_3)$ is of the above form and $a_1$ represents a nonnegative integer.

    \item The third step extracts the units digit of the two-digit factor and produces the elementary multiplication
    \[
        D_4^{(3)}(x,q_1,a_1,q_2,a_2)
        \coloneqq
        d_1\cdot d_3,
    \]
    whenever $x=d_1\cdot(10d_2+d_3)$ is of the above form.

    \item The fourth step adds the two partial products.
    For generic coordinates $(x,q_1,a_1,q_2,a_2,q_3,a_3)\in\cX^7$, define
    \[
        D_4^{(4)}(x,q_1,a_1,q_2,a_2,q_3,a_3)
        \coloneqq
        a_2+a_3,
    \]
    whenever $x=d_1\cdot(10d_2+d_3)$ is of the above form and $a_2,a_3$ represent nonnegative integers.
\end{enumerate}
On all remaining inputs, including cases where one of the coordinates required above does not represent a nonnegative integer, define the corresponding step arbitrarily, for example by returning the fixed element $0\in\cX$.
This makes each $D_4^{(k)}$ a well-defined function from $\cX^{2k-1}$ to $\cX$.
The precise behavior of the chain rule outside this illustrative family is irrelevant for the example.

\textbf{Recoverability of the illustrative family.}

For this illustrative chain rule, every expression of the form
$d_1\cdot(10d_2+d_3)$ with $d_1,d_3\in\set{0,\dots,9}$ and $d_2\in\set{1,\dots,9}$ is $(D_4,g)$-recoverable.
Indeed, along the oracle trajectory, we have
\[
    Q^{(1)}_{D_4,g}(x)
    =
    d_1\cdot d_2,
    \qquad
    A^{(1)}_{D_4,g}(x)
    =
    g(d_1\cdot d_2)
    =
    d_1d_2.
\]
The second step then gives
\[
    Q^{(2)}_{D_4,g}(x)
    =
    10\cdot A^{(1)}_{D_4,g}(x)
    =
    10\cdot d_1d_2,
    \qquad
    A^{(2)}_{D_4,g}(x)
    =
    g(10\cdot d_1d_2)
    =
    10d_1d_2.
\]
The third step gives
\[
    Q^{(3)}_{D_4,g}(x)
    =
    d_1\cdot d_3,
    \qquad
    A^{(3)}_{D_4,g}(x)
    =
    g(d_1\cdot d_3)
    =
    d_1d_3.
\]
Finally,
\[
    Q^{(4)}_{D_4,g}(x)
    =
    A^{(2)}_{D_4,g}(x)+A^{(3)}_{D_4,g}(x)
    =
    10d_1d_2+d_1d_3,
\]
and therefore
\[
    A^{(4)}_{D_4,g}(x)
    =
    g\brb{10d_1d_2+d_1d_3}
    =
    d_1(10d_2+d_3)
    =
    g(x).
\]
Thus $x$ is $(D_4,g)$-recoverable.

This example uses $g$ as the ideal ground-truth evaluator only to specify correctness.
It does not mean that a learned hypothesis must be trained on all multiplication expressions.
For instance, the source distribution may be supported only on elementary expressions such as one-digit-by-one-digit multiplications, multiplications by $10$, and additions, while the test distribution may be supported on one-digit-by-two-digit multiplication expressions.
The chain rule then illustrates how CoT can transform a test query into elementary subquestions on which the hypothesis may already be accurate.

\subsection{Recoverability and the support assumption}
    \label[appendix]{sec:large-risks-no-structure}
    In Section \ref{sec:formalizing-cot}, we assumed that the distribution $\nu$ is supported on the $(D_K, g)$-recoverable set, so that we can decompose the reasoning risk into the two terms in Inequality~\eqref{eq:canonical-decomposition}. In the absence of this structural assumption, the reasoning risk $R_{\nu,\ell}^{D_K}(f,g)$ is bounded by the sum of three terms, rather than by the two terms appearing in \eqref{eq:canonical-decomposition}:
        \begin{align*}
            R_{\nu,\ell}^{D_K}(f,g) \leq &\underbrace{
            \underset{X \sim \nu}{\E}
            \bbsb{
                \ell\Brb{
                    f\brb{ Q_{D_K,f}^{(K)}(X) }, \,
                    f\brb{ Q_{D_K,g}^{(K)}(X) }
                }
            }
        }_{\text{trajectory-mismatch risk (TMR)}}
        +
        \underbrace{
            \underset{X \sim \nu}{\E}
            \bbsb{
                \ell\Brb{
                    f \brb{ Q_{D_K,g}^{(K)}(X) }, \, 
                    g \brb{ Q_{D_K,g}^{(K)}(X) }
                }
            }
            }_{\text{oracle-trajectory risk (OTR)}} + \\ & \underbrace{
            \underset{X \sim \nu}{\E}
            \bbsb{
                \ell\Brb{
                    g \brb{ Q_{D_K,g}^{(K)}(X) }, \, 
                    g \brb{ X }
                }
            }
            }_{\text{oracle-mismatch risk (OMR)}}
    \end{align*}
    In this decomposition, an additional term emerges—the oracle-mismatch risk (OMR)—which quantifies the extent of misalignment between the chain rule $D_K$ and the oracle $g$. Notably, this term is unbounded in general and can be arbitrarily large, as established in the following proposition.
    
    \begin{proposition}
        For any integer $K\geq 2$ and any $M>0$, there exist a metric space $(\mc{X}, \rho)$, a loss function $\ell\colon \mc{X}^2\to [0,\infty)$, two functions $f,g\colon \mc{X}\to\mc{X}$, a $K$-step chain rule $D_K$, and a distribution $\nu$ whose support is not contained in the $(D_K,g)$-recoverable set, such that both TMR and OTR achieve $0$, while the OMR is exactly $M$,
        \begin{align*}
            \underset{X \sim \nu}{\E}
            \bbsb{
                \ell\Brb{
                    f\brb{ Q_{D_K,f}^{(K)}(X) }, \,
                    f\brb{ Q_{D_K,g}^{(K)}(X) }
                }
            } &= 0 \\
            \underset{X \sim \nu}{\E}
            \bbsb{
                \ell\Brb{
                    f \brb{ Q_{D_K,g}^{(K)}(X) }, \, 
                    g \brb{ Q_{D_K,g}^{(K)}(X) }
                }
            } &= 0 \\
            \underset{X \sim \nu}{\E}
            \bbsb{
                \ell\Brb{
                    g \brb{ Q_{D_K,g}^{(K)}(X) }, \, 
                    g \brb{ X }
                }
            } & = M
        \end{align*}
    \end{proposition}

    \begin{proof}
    Let $\cX=\R$. 
    For any $a,b\in\cX$, let the metric be $\rho(a,b)\coloneqq |a-b|$ and the loss be $\ell(a,b)\coloneqq |a-b|$. 
    Define the answer maps $f,g\colon\cX\to\cX$ by
    \[
        f(x)\coloneqq x,
        \qquad
        g(x)\coloneqq |x|
        \qquad
        \text{for every } x\in\cX .
    \]
    Let $\nu$ be the uniform distribution on $[0,L]$, for some $L>0$.

    For the chain rule $D_K=\brb{D_K^{(1)},\ldots,D_K^{(K)}}$, define
    \[
        D_K^{(1)}(x)\coloneqq x+M,
    \]
    and, for every $k\in\{2,\ldots,K\}$,
    \[
        D_K^{(k)}
        \brb{x,x^{(1)},y^{(1)},\ldots,x^{(k-1)},y^{(k-1)}}
        \coloneqq
        y^{(k-1)} .
    \]

    We first compute the trajectories pointwise. 
    Fix any $x\in[0,L]$. 
    Since $x\ge 0$ and $x+M>0$, we have
    \[
        Q_{D_K,f}^{(1)}(x)
        =
        Q_{D_K,g}^{(1)}(x)
        =
        x+M,
    \]
    and therefore
    \[
        A_{D_K,f}^{(1)}(x)
        =
        f(x+M)
        =
        x+M
        =
        g(x+M)
        =
        A_{D_K,g}^{(1)}(x).
    \]
    By induction, for every $k\in\{1,\ldots,K\}$,
    \[
        Q_{D_K,f}^{(k)}(x)
        =
        Q_{D_K,g}^{(k)}(x)
        =
        x+M,
        \qquad
        A_{D_K,f}^{(k)}(x)
        =
        A_{D_K,g}^{(k)}(x)
        =
        x+M .
    \]
    Indeed, once the first answer is $x+M$, every subsequent step of the chain rule returns the previous answer, and both $f$ and $g$ act as the identity on the nonnegative point $x+M$.

    It follows that, for every $x\in[0,L]$,
    \[
        f\brb{Q_{D_K,f}^{(K)}(x)}
        =
        f\brb{Q_{D_K,g}^{(K)}(x)}
        =
        x+M,
    \]
    and hence the TMR integrand is zero. 
    Similarly, since $Q_{D_K,g}^{(K)}(x)=x+M\ge 0$,
    \[
        f\brb{Q_{D_K,g}^{(K)}(x)}
        =
        x+M
        =
        g\brb{Q_{D_K,g}^{(K)}(x)},
    \]
    and hence the OTR integrand is zero.

    On the other hand, for every $x\in[0,L]$,
    \[
        g\brb{Q_{D_K,g}^{(K)}(x)}
        =
        g(x+M)
        =
        x+M,
        \qquad
        g(x)=x,
    \]
    and therefore
    \[
        \ell\Brb{
            g\brb{Q_{D_K,g}^{(K)}(x)},
            g(x)
        }
        =
        \ell(x+M,x)
        =
        M .
    \]
    Thus, if $X\sim\nu$, then
    \[
    \underset{X\sim\nu}{\E}
    \bbsb{
        \ell\Brb{
            f\brb{Q_{D_K,f}^{(K)}(X)},
            f\brb{Q_{D_K,g}^{(K)}(X)}
        }
    }
    =
    0,
    \]
    \[
    \underset{X\sim\nu}{\E}
    \bbsb{
        \ell\Brb{
            f\brb{Q_{D_K,g}^{(K)}(X)},
            g\brb{Q_{D_K,g}^{(K)}(X)}
        }
    }
    =
    0,
    \]
    and
    \[
    \underset{X\sim\nu}{\E}
    \bbsb{
        \ell\Brb{
            g\brb{Q_{D_K,g}^{(K)}(X)},
            g(X)
        }
    }
    =
    M .
    \]

    Finally, the support of $\nu$ is not contained in the $(D_K,g)$-recoverable set. 
    Indeed, for every $x\in[0,L]$,
    \[
        A_{D_K,g}^{(K)}(x)=x+M
        \neq
        x
        =
        g(x),
    \]
    so no point in $[0,L]$ is $(D_K,g)$-recoverable.
\end{proof}
    
    Therefore, to control the oracle-mismatch risk (OMR) and prevent it from becoming arbitrarily large, it is natural to impose structural assumptions on the distribution $\nu$. A natural choice is to assume that $\nu$ is supported on the $(D_K, g)$-recoverable set, under which the otherwise uncontrollable OMR term vanishes.

\section{Additional material for Section~\ref{sec:benefit-cot}}
\label[appendix]{sec:OTR_bounds}

As discussed in Section~\ref{sec:benefit-cot}, the oracle-trajectory risk (OTR) in Inequality~\eqref{eq:canonical-decomposition} can be controlled using tools from domain adaptation.
We spell out one standard way to obtain such a bound.
The goal is to relate the target-domain risk under the oracle-trajectory distribution $Q_{D_K,g}^{(K)}\#\nu$ to the empirical risk under the source distribution $\mu$, plus a discrepancy term and standard complexity terms.

We first recall the notion of empirical Rademacher complexity.
Let $\mathcal{G}$ be a class of real-valued functions on a set $\mc{S}$.
For a sample $S=(S_1,\ldots,S_m)\in\mc{S}^m$, define
\[
    \hat{\mathfrak{R}}_S(\mathcal{G})
    \coloneqq
    2\underset{\sigma_1,\ldots,\sigma_m\sim\mathrm{Uni}(\set{-1,1}^m)}{\E}
    \bbsb{
        \sup_{\gamma\in\mathcal{G}}
        \left|
        \frac{1}{m}\sum_{i=1}^m \sigma_i \gamma(S_i)
        \right|
    } .
\]
If $\mathcal{D}$ is a distribution on $\mc{S}$, the Rademacher complexity of $\mathcal{G}$ at sample size $m$ is
\[
    \mathfrak{R}_m(\mathcal{G})
    \coloneqq
    \underset{S\sim\mathcal{D}^m}{\E}
    \bbsb{
        \hat{\mathfrak{R}}_S(\mathcal{G})
    } .
\]

From this point onward, assume that hypotheses $f$ are drawn from a hypothesis class $\mathcal{H}\subseteq\mathcal{X}^{\mathcal{X}}$.
To quantify the discrepancy between the source and target distributions, we use the $d^\ell_{\mathcal{H}}$-divergence~\citep{mansour_domain_2023}, a loss-dependent notion of distributional distance that generalizes the classical $\mathcal{H}$-divergence beyond the $0$--$1$ loss setting~\citep{ben-david_analysis_2006,ben-david_theory_2010}.

\begin{definition}[$d^\ell_\mathcal{H}$-divergence]
Given a loss function $\ell\colon\mc{X}\times\mc{X}\to[0,\infty)$ and two distributions $\mu$ and $\nu$ on $\mc{X}$, define
\[
    d^\ell_\mathcal{H}(\mu,\nu)
    \coloneqq
    \sup_{h,h'\in\mathcal{H}}
    \left|
    \underset{x\sim\mu}{\E}
    \bbsb{
        \ell\brb{h(x),h'(x)}
    }
    -
    \underset{x\sim\nu}{\E}
    \bbsb{
        \ell\brb{h(x),h'(x)}
    }
    \right| .
\]
\end{definition}

The quantity $d^\ell_\mathcal{H}$ is symmetric and satisfies the triangle inequality.
Its empirical counterpart is defined analogously.
Given independent samples $S=(S_1,\ldots,S_m)\sim\mu^m$ and $T=(T_1,\ldots,T_m)\sim\nu^m$, let $\hat{\mu}_S$ and $\hat{\nu}_T$ denote the corresponding empirical distributions, and set
\[
    \hat{d}^{\ell,m}_\mathcal{H}(\mu,\nu)
    \coloneqq
    d^\ell_\mathcal{H}\brb{\hat{\mu}_S,\hat{\nu}_T}.
\]
Equivalently,
\[
    \hat{d}^{\ell,m}_\mathcal{H}(\mu,\nu)
    =
    \sup_{h,h'\in\mathcal{H}}
    \left|
    \frac{1}{m}\sum_{i=1}^m \ell\brb{h(S_i),h'(S_i)}
    -
    \frac{1}{m}\sum_{i=1}^m \ell\brb{h(T_i),h'(T_i)}
    \right| .
\]

We shall use the following standard Rademacher bounds.
For a distribution $\mu$ on $\mc{X}$ and a sample $S=(X_1,\ldots,X_m)\sim\mu^m$, write
\[
    \hat{R}_{S,\ell}(h,g)
    \coloneqq
    \frac{1}{m}\sum_{i=1}^m \ell\brb{h(X_i),g(X_i)} .
\]

\begin{lemma}
\label{lem:radbound-risks}
Let $\ell\colon\mc{X}\times\mc{X}\to[0,M]$ be a bounded loss function, and define
\[
    \mathcal{F}_{\mathcal{H}}
    \coloneqq
    \set{(x,y)\mapsto \ell\brb{h(x),y}\colon h\in\mathcal{H}} .
\]
Let $S=((X_1,g(X_1)),\ldots,(X_m,g(X_m)))$, where $X_1,\ldots,X_m$ are drawn i.i.d. from $\mu$.
Then, for any $\epsilon>0$, with probability at least $1-\epsilon$ over $S$, the following inequality holds for all $h\in\mathcal{H}$:
\[
    R_{\mu,\ell}(h,g)
    \le
    \hat{R}_{S,\ell}(h,g)
    +
    \hat{\mathfrak{R}}_S(\mathcal{F}_{\mathcal{H}})
    +
    3M\sqrt{\frac{\log\frac{2}{\epsilon}}{2m}} .
\]
\end{lemma}

\begin{lemma}
\label{lemma:radbound-divergence}
Let $\ell\colon\mc{X}\times\mc{X}\to[0,M]$ be a bounded loss function, and define
\[
    L_{\mathcal{H}}
    \coloneqq
    \set{x\mapsto \ell\brb{h(x),h'(x)}\colon h,h'\in\mathcal{H}} .
\]
Let $\eta$ be a distribution on $\mc{X}$, let $S=(X_1,\ldots,X_m)\sim\eta^m$, and let $\hat{\eta}_S$ be the empirical distribution induced by $S$.
Then, for any $\epsilon>0$, with probability at least $1-\epsilon$ over $S$,
\[
    d^\ell_\mathcal{H}(\eta,\hat{\eta}_S)
    \le
    \hat{\mathfrak{R}}_S(L_\mathcal{H})
    +
    3M\sqrt{\frac{\log\frac{2}{\epsilon}}{2m}} .
\]
\end{lemma}

Combining Lemma~\ref{lemma:radbound-divergence} with the symmetry and triangle inequality of $d^\ell_\mathcal{H}$ gives the following empirical-discrepancy bound.

\begin{lemma}
\label{lem:radbound-div}
Let $\ell\colon\mc{X}\times\mc{X}\to[0,M]$ be a bounded loss function.
Let $\mu$ and $\nu$ be distributions on $\mc{X}$, and let $S\sim\mu^m$ and $T\sim\nu^m$ be independent samples.
Then, for any $\epsilon>0$, with probability at least $1-\epsilon$ over $S$ and $T$,
\[
    d^\ell_\mathcal{H}(\mu,\nu)
    \le
    \hat{d}^{\ell,m}_\mathcal{H}(\mu,\nu)
    +
    \hat{\mathfrak{R}}_S(L_\mathcal{H})
    +
    \hat{\mathfrak{R}}_T(L_\mathcal{H})
    +
    6M\sqrt{\frac{\log\frac{4}{\epsilon}}{2m}} .
\]
\end{lemma}

\begin{proof}
By the triangle inequality and symmetry of $d^\ell_\mathcal{H}$,
\[
    d^\ell_\mathcal{H}(\mu,\nu)
    \le
    d^\ell_\mathcal{H}(\mu,\hat{\mu}_S)
    +
    d^\ell_\mathcal{H}(\hat{\mu}_S,\hat{\nu}_T)
    +
    d^\ell_\mathcal{H}(\hat{\nu}_T,\nu).
\]
The middle term is $\hat{d}^{\ell,m}_\mathcal{H}(\mu,\nu)$.
Applying Lemma~\ref{lemma:radbound-divergence} to the first and third terms with failure probability $\epsilon/2$ each gives the claim.
\end{proof}

We now derive a Rademacher-based bound for the OTR.
Let
\[
    \tau
    \coloneqq
    Q_{D_K,g}^{(K)}\#\nu
\]
denote the oracle-trajectory target distribution.
For a quasimetric loss, define the approximation term
\[
    \beta_{\mu,\tau}
    \coloneqq
    \inf_{h\in\mathcal{H}}
    \Brb{
        R_{\mu,\ell}(g,h)
        +
        R_{\tau,\ell}(h,g)
    } .
\]
If $\ell$ is symmetric, this reduces to the more familiar expression
$\inf_{h\in\mathcal{H}}\Brb{R_{\mu,\ell}(h,g)+R_{\tau,\ell}(h,g)}$.

\begin{theorem}
\label{thm:otr-rademacher-bound}
Let $K\ge 1$, and let $\ell\colon\mc{X}\times\mc{X}\to[0,M]$ be a bounded quasimetric loss.
Let $\tau\coloneqq Q_{D_K,g}^{(K)}\#\nu$.
Let $S=((X_1,g(X_1)),\ldots,(X_m,g(X_m)))$ with $X_i\sim\mu$ i.i.d., let $U=(U_1,\ldots,U_m)\sim\mu^m$, and let $T=(T_1,\ldots,T_m)\sim\tau^m$, with all samples independent.
Then, for any $\epsilon>0$, with probability at least $1-\epsilon$ over $S,U,T$, the following bound holds for every $f\in\mathcal{H}$:
\[
    R_{\tau,\ell}(f,g)
    \le
    \hat{R}_{S,\ell}(f,g)
    +
    \hat{d}^{\ell,m}_\mathcal{H}(\mu,\tau)
    +
    \hat{\mathfrak{R}}_S(\mathcal{F}_{\mathcal{H}})
    +
    \hat{\mathfrak{R}}_U(L_\mathcal{H})
    +
    \hat{\mathfrak{R}}_T(L_\mathcal{H})
    +
    9M\sqrt{\frac{\log\frac{6}{\epsilon}}{2m}}
    +
    \beta_{\mu,\tau}.
\]
Equivalently,
\[
    R_{Q_{D_K,g}^{(K)}\#\nu,\ell}(f,g)
    \le
    \hat{R}_{S,\ell}(f,g)
    +
    \hat{d}^{\ell,m}_\mathcal{H}\brb{\mu,Q_{D_K,g}^{(K)}\#\nu}
    +
    \hat{\mathfrak{R}}_S(\mathcal{F}_{\mathcal{H}})
    +
    \hat{\mathfrak{R}}_U(L_\mathcal{H})
    +
    \hat{\mathfrak{R}}_T(L_\mathcal{H})
    +
    9M\sqrt{\frac{\log\frac{6}{\epsilon}}{2m}}
    +
    \beta_{\mu,Q_{D_K,g}^{(K)}\#\nu}.
\]
\end{theorem}

\begin{proof}
Fix any $h\in\mathcal{H}$.
Using the triangle inequality for $\ell$, we have
\[
    R_{\tau,\ell}(f,g)
    \le
    \underset{x\sim\tau}{\E}\bbsb{\ell\brb{f(x),h(x)}}
    +
    R_{\tau,\ell}(h,g).
\]
By the definition of $d^\ell_\mathcal{H}$,
\[
    \underset{x\sim\tau}{\E}\bbsb{\ell\brb{f(x),h(x)}}
    \le
    \underset{x\sim\mu}{\E}\bbsb{\ell\brb{f(x),h(x)}}
    +
    d^\ell_\mathcal{H}(\mu,\tau).
\]
Applying the triangle inequality again gives
\[
    \underset{x\sim\mu}{\E}\bbsb{\ell\brb{f(x),h(x)}}
    \le
    R_{\mu,\ell}(f,g)
    +
    R_{\mu,\ell}(g,h).
\]
Combining the previous three displays and then taking the infimum over $h\in\mathcal{H}$ yields
\[
    R_{\tau,\ell}(f,g)
    \le
    R_{\mu,\ell}(f,g)
    +
    d^\ell_\mathcal{H}(\mu,\tau)
    +
    \beta_{\mu,\tau}.
\]
It remains to replace the population source risk and the population discrepancy by empirical quantities.
Apply Lemma~\ref{lem:radbound-risks} to $S$, and apply Lemma~\ref{lemma:radbound-divergence} separately to $U\sim\mu^m$ and $T\sim\tau^m$, each with failure probability $\epsilon/3$.
A union bound gives probability at least $1-\epsilon$, and the three deviation terms sum to
\[
    9M\sqrt{\frac{\log\frac{6}{\epsilon}}{2m}} .
\]
This proves the claim.
\end{proof}

This bound relates the OTR to the empirical source risk, a discrepancy between the source distribution $\mu$ and the oracle-trajectory target distribution $Q_{D_K,g}^{(K)}\#\nu$, and standard Rademacher complexity terms.
Thus, it makes explicit how statistical complexity and distribution shift jointly contribute to the benefit term in the CoT risk decomposition.

    \newpage

        \end{document}